\documentclass{article}
\usepackage[final]{neurips_2024}

\usepackage{hyperref}
\usepackage{url}   
\usepackage{empheq}

\usepackage[colorinlistoftodos,prependcaption,textsize=small]{todonotes}

\usepackage{algorithmic}
\usepackage[ruled,vlined]{algorithm2e}
\usepackage{amsmath}

\usepackage[acronym,nowarn,section,nogroupskip,nonumberlist]{glossaries}
\usepackage{xcolor}
\usepackage{color}
\usepackage{graphicx}

\usepackage{silence}
\usepackage[utf8]{inputenc}
\usepackage[T1]{fontenc}
\usepackage{tabularx}
\usepackage{lipsum}
\usepackage{url}
\usepackage{placeins}
\usepackage{booktabs}
\usepackage{soul}
\usepackage{amsfonts}
\usepackage{nicefrac}
\usepackage{microtype}
\usepackage{wrapfig}
\usepackage{caption}
\usepackage{amssymb}
\usepackage{pifont}

\usepackage{subcaption}
\usepackage{float}
\usepackage{rotating}
\usepackage{etoolbox}
\usepackage{xparse}
\usepackage{grffile}
\usepackage{amsmath}
\usepackage{xspace}
\usepackage{euscript}
\usepackage{enumitem}
\usepackage{mathtools}
\usepackage{tikz}
\usepackage{tablefootnote}
\usepackage[flushleft]{threeparttable}
\usepackage{multirow}
\usepackage[title]{appendix}
\usepackage{arydshln}
\usepackage{array, booktabs}
\usepackage{comment}
\usepackage{graphicx}
\usepackage{adjustbox}
\usepackage{amsmath,amsthm}
\usepackage{dsfont}
\usepackage{balance}
\usepackage{multicol}
\usepackage{xcolor}
\usepackage{nameref}
\usepackage{pdfpages}
\usepackage{braket}
\usepackage[normalem]{ulem}
\usepackage{bbding}
\usepackage[capitalise, noabbrev, nameinlink]{cleveref}   
\usepackage{xfrac}
\usepackage[usestackEOL]{stackengine}
\usepackage{indentfirst}
\usepackage{csquotes}
\usepackage{pgf}
\usepackage{varwidth}
\usepackage{verbatimbox}
\usepackage{verbatim}
\usepackage{bm}
\usepackage{anyfontsize}

\usepackage{empheq}
\definecolor{myblue}{rgb}{.8, .8, 1}

\definecolor{pastelblue}{RGB}{76,113,175}
\definecolor{pastelgreen}{RGB}{144,238,144}
\definecolor{pastelred}{RGB}{196,78,82}
\definecolor{pastelgrey}{RGB}{230,230,230}
\definecolor{pastelbeige}{RGB}{243,236,221}
\definecolor{pastelpurple}{RGB}{154,139,192}
\definecolor{salmon}{RGB}{250, 128, 114}
\definecolor{darkgreen}{rgb}{0,0.6,0}
\definecolor{darkred}{rgb}{0.5,0,0}
\definecolor{verylightgreen}{HTML}{F6FFF9}
\definecolor{verylightred}{HTML}{FFF4F3}
\definecolor{verylightgray}{HTML}{F4F6F6}
\definecolor{babyblueeyes}{rgb}{0.63, 0.79, 0.95}
\definecolor{lightpink}{rgb}{1.00, 0.714, 0.757}

\usepackage{tikzpeople}
\usetikzlibrary{shapes,decorations,arrows,calc,arrows.meta,fit,positioning}

\tikzset{
    -Latex,auto,node distance =1 cm and 1 cm,semithick,
    state/.style ={ellipse, draw, minimum width = 0.7 cm},
    point/.style = {circle, draw, inner sep=0.04cm,fill,node contents={}},
    bidirected/.style={Latex-Latex,dashed},
    el/.style = {inner sep=2pt, align=left, sloped}
}

\newtheorem{theorem}{Theorem}

\makeatletter
\def\thmt@refnamewithcomma #1#2#3,#4,#5\@nil{%
	\@xa\def\csname\thmt@envname #1utorefname\endcsname{#3}%
	\ifcsname #2refname\endcsname
	\csname #2refname\expandafter\endcsname\expandafter{\thmt@envname}{#3}{#4}%
	\fi}
\makeatother
\Crefname{conjecture}{Conjecture}{Conjectures}
\Crefname{definition}{Definition}{Definitions}
\Crefname{observation}{Observation}{Observations}
\Crefname{assumption}{Assumption}{Assumptions}
\Crefname{axiom}{Axiom}{Axioms}
\Crefname{case}{Case}{Cases}
\Crefname{claim}{Claim}{Claims}
\Crefname{conclusion}{Conclusion}{Conclusions}
\Crefname{condition}{Condition}{Conditions}
\Crefname{criterion}{Criterion}{Criteria}
\Crefname{exercise}{Exercise}{Exercises}
\Crefname{example}{Example}{Examples}
\Crefname{notation}{Notation}{Notations}
\Crefname{problem}{Problem}{Problems}
\Crefname{property}{Property}{Properties}
\Crefname{remark}{Remark}{Remarks}
\Crefname{solution}{Solution}{Solutions}
\Crefname{summary}{Summary}{Summaries}
\Crefname{motivation}{Motivation}{Motivations}
\Crefname{query}{Query}{Queries}

\newcommand*\dbar[1]{\overline{\overline{\lower0.2ex\hbox{$#1$}}}}

\def\cB{{\mathcal{B}}}

\def\cL{{\mathcal{L}}}

\def\cN{{\mathcal{N}}}

\def\cQ{{\mathcal{Q}}}

\def\cS{{\mathcal{S}}}
\def\cT{{\mathcal{T}}}
\def\cU{{\mathcal{U}}}

\DeclareFontFamily{U}{BOONDOX-calo}{\skewchar\font=45 }
\DeclareFontShape{U}{BOONDOX-calo}{m}{n}{
  <-> s*[1.05] BOONDOX-r-calo}{}
\DeclareFontShape{U}{BOONDOX-calo}{b}{n}{
  <-> s*[1.05] BOONDOX-b-calo}{}
\DeclareMathAlphabet{\mathcalb}{U}{BOONDOX-calo}{m}{n}
\SetMathAlphabet{\mathcalb}{bold}{U}{BOONDOX-calo}{b}{n}
\DeclareMathAlphabet{\mathbcalb}{U}{BOONDOX-calo}{b}{n}

\renewcommand{\paragraph}[1]{{\noindent \textbf{#1.}}}

\let\originalleft\left
\let\originalright\right
\renewcommand{\left}{\mathopen{}\mathclose\bgroup\originalleft}
\renewcommand{\right}{\aftergroup\egroup\originalright}

\global\long\def\S{\mathcal{S}}

\global\long\def\inner#1#2{\left\langle \vphantom{1^2} #1, #2\right\rangle}

\definecolor{antiquefuchsia}{rgb}{0.57, 0.36, 0.51}
\definecolor{amethyst}{rgb}{0.6, 0.4, 0.8}

\newcommand{\deriv}[2]{\frac{\partial #1}{\partial #2}}
\newcommand{\dderiv}[2]{\dfrac{\partial #1}{\partial #2}}

\newcommand{\var}{{\rm I\kern-.3em D}}

\newcommand{\cond}{\,|\,}
\newcommand{\Normal}{\mathcal{N}}

\newcommand{\eps}{\varepsilon}

\newtheorem*{theorem*}{Theorem}
\newtheorem*{proposition*}{Proposition}
\newtheorem*{example*}{Example}
\DeclareMathSymbol{\shortminus}{\mathbin}{AMSa}{"39}

\newcommand{\Dt}{G_t}
\newcommand{\bbP}{\mathbb{P}}
\newcommand{\bbPref}{\mathbb{P}^{\text{ref}}}
\newcommand{\bbQ}{\mathbb{Q}}

\crefname{equation}{Eq.}{Eqs.}
\crefname{section}{Sec.}{Secs.}
\crefname{corollary}{Cor.}{Cors.}
\crefname{proposition}{Prop.}{Props.}
\crefname{appendix}{App.}{Apps.}
\crefname{theorem}{Thm.}{Thms.}
\crefname{figure}{Fig.}{Figs.}
\crefname{tabular}{Tab.}{Tabs.}
\crefname{algorithm}{Alg.}{Algs.}
\crefname{alg}{Alg.}{Algs.}

\newcommand\sbullet[1][.5]{\mathbin{\vcenter{\hbox{\scalebox{#1}{$\bullet$}}}}}

\newacronym{SB}{SB}{Schrödinger Bridge}
\newacronym{FPE}{FPE}{Fokker-Planck equation}
\newacronym{SDE}{SDE}{stochastic differential equation}
\newacronym{MCMC}{MCMC}{Markov Chain Monte Carlo}
\newacronym{MH}{MH}{Metropolis-Hastings}
\newacronym{MD}{MD}{molecular dynamics}
\newacronym{TPS}{TPS}{\textit{transition path sampling}}
\newacronym{RMSD}{RMSD}{root mean square deviation}

\newcommand{\paraorsubsec}[1]{\paragraph{#1}}
\newcommand{\condt}{{t|0,T}}

\newcommand{\ifbm}[1]
{#1}
\newcommand{\btofxt}{\ifbm{b_t}(\ifbm{x_t})}

\newcommand{\tpsx}{\ifbm{x}}

\newcommand{\tpsy}{\ifbm{y}}

\newcommand{\prob}{\rho}
\newcommand{\condp}{{\prob}}

\newcommand{\setB}{\cB}
\newcommand{\mass}{M}

\newcommand{\hset}{h_{\setB}}
\newcommand{\hpt}{h_{B}}

\newcommand{\tpsxbar}{\bar{x}} 
 \newcommand{\mom}{\bar{v}}

\newcommand{\vq}{v_\condt^{(q,\theta)}}
\newcommand{\vqtheta}{v_\condt^{(q,\theta)}}
\newcommand{\uq}{u_\condt^{(q,\theta)}}

\newcommand{\xdim}{D}
\newcommand{\doobdim}{N}
\newcommand{\xdomain}{\mathbb{R}^{\xdim}}

\newcommand{\bbR}{\mathbb{R}}
\newcommand{\bbI}{\mathbb{I}}

\newcommand{\secvheader}{\vspace*{-.2cm}}
\usepackage{thm-restate}

\definecolor{cite_color}{HTML}{114083}
\definecolor{url_color}{RGB}{153, 102,  0}
\hypersetup{
 colorlinks = true,
 citecolor = cite_color,
 linkcolor = url_color,
 urlcolor = black
}

\usepackage{cancel}
\glsdisablehyper

\newlist{myenum}{enumerate}{1}
\setlist[myenum,1]{label={\arabic*.},
                   ref  ={\arabic*}}
\crefname{myenumi}{Challenge}{Challenge}
\crefformat{subequation}{#2(#1#3)}
\crefrangeformat{subequation}{#3(#1#4,#5#2#6)}
\crefmultiformat{subequation}{#2(#1#3)}{, #2(#1#3)}{, #2(#1#3)}{, #2(#1#3)}

\usepackage{relsize}

\usepackage[theorems,most]{tcolorbox}
\newtcolorbox{mymath}[1][]{%
    nobeforeafter,  tcbox raise base,
    colframe=white!30!black,
    colback=olive!3,
    boxrule=0.2mm,
    boxsep=5pt,
    #1}
\newenvironment{mytheorem}{\begin{theorem}}{\end{theorem}}

\tcolorboxenvironment{mytheorem}{%
    nobeforeafter,  tcbox raise base,
    colframe=white!30!black,
    colback=olive!3,
    boxrule=0.2mm,
  left=0pt,
  right=0pt,
  boxsep=2.5pt,
}


\def\ourCode{{\href{https://github.com/plainerman/variational-doob}{\texttt{https://github.com/plainerman/variational-doob}}}} 

\title{Doob's Lagrangian: A Sample-Efficient Variational Approach to Transition Path Sampling}

\author{
  Yuanqi Du\thanks{Equal contribution. ~
  Correspondence: \texttt{\href{mailto:k.necludov@gmail.com}{k.necludov@gmail.com} (Kirill Neklyudov)}}\,\,~$^{1}$ \And Michael Plainer\footnotemark[1]\,\,~$^{2,3,4,5}$ \And Rob Brekelmans\footnotemark[1]\,\,~$^{6}$ \And Chenru Duan~$^{7,8}$ \And Frank Noé~$^{4,9,10}$ \And Carla P. Gomes~$^{1}$ \And Alán Aspuru-Guzik~$^{6,11}$ \And Kirill Neklyudov~
  $^{12,13}$ \\ 
  \And \textnormal{
  $^{1}$Cornell University \quad
  $^{2}$Zuse School ELIZA \quad
  $^{3}$Technische Universität Berlin
  } \\
  \textnormal{
  $^{4}$Freie Universität Berlin \quad
  $^{5}$Berlin Institute for the Foundations of Learning and Data
  } \\
  \textnormal{
  $^{6}$Vector Institute \quad
  $^{7}$Massachusetts Institute of Technology \quad
  $^{8}$Deep Principle, Inc. \quad
  }\\
  \textnormal{
  $^{9}$Rice University \quad
  $^{10}$Microsoft Research AI4Science \quad
  $^{11}$University of Toronto
  }\\
  \textnormal{
  $^{12}$Université de Montréal \quad
  $^{13}$Mila Quebec AI Institute}
}

\begin{document}
\maketitle

\begin{abstract}
Rare event sampling in dynamical systems is a fundamental problem arising in the natural sciences, which poses significant computational challenges due to an exponentially large space of trajectories. For settings where the dynamical system of interest follows a Brownian motion with known drift, the question of conditioning the process to reach a given endpoint or desired rare event is definitively answered by Doob's $h$-transform. However, the naive estimation of this transform is infeasible, as it requires simulating sufficiently many forward trajectories to estimate rare event probabilities. In this work, we propose a variational formulation of Doob's $h$-transform as an optimization problem over trajectories between a given initial point and the desired ending point. To solve this optimization, we propose a simulation-free training objective with a model parameterization that imposes the desired boundary conditions by design. Our approach significantly reduces the search space over trajectories and avoids expensive trajectory simulation and inefficient importance sampling estimators which are required in existing methods. We demonstrate the ability of our method to find feasible transition paths on real-world molecular simulation and protein folding tasks.
\end{abstract}

\secvheader
\secvheader
\section{Introduction}
\secvheader

Conditioning a stochastic process to 
obey a particular endpoint distribution, satisfy desired terminal conditions, or observe a rare event is a problem with a long history \citep{schrodinger1932theorie, doob1957conditional} and wide-ranging applications from generative modeling \citep{de2021diffusion, chen2021likelihood, liu2022let, liu2023learning, somnath2023aligned} to molecular simulation \citep{anderson2007predicting, wu2022diffusion, plainer2023transition, holdijk2024stochastic}, drug discovery ~\citep{kirmizialtin2012how, kirmizialtin2015enzyme, dickson2018}, and materials science ~\citep{xi2013, selli2016hierarchical, dilawar2016}.

\paragraph{Transition Path Sampling}
In this work, we take a particular interest in the problem of \gls{TPS} in computational chemistry \citep{dellago2002transition, weinan2010transition}, which attempts to describe how molecules transition between local energy minima or metastable states under random fluctuations or the influence of external forces. Understanding such transitions has numerous applications for combustion, catalysis, battery, material design, and protein folding \citep{zeng2020complex,klucznik2024computational,blau2021chemically,noe2009constructing,escobedo2009}. While the TPS problem is often framed as finding the `most probable path' transitioning between states \citep{durr1978onsager, vanden2008geometric}, we build upon connections between TPS and Doob's $h$-transform \citep{das2019variational,  das2021reinforcement, das2022direct, koehl2022sampling, singh2023variational} and seek to match the \textit{full} posterior distribution over conditioned processes.

\begin{figure}[t]\vspace*{-.3cm}
    \centering
    \includegraphics[width=\linewidth]{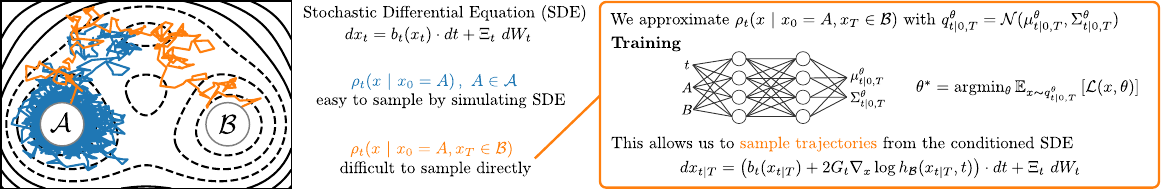}
    \caption{
     Given reference dynamics, transition path sampling seeks to capture the conditional or posterior distribution over paths which reach a terminal set $x_T \in \mathcal{B}$.    However, simulating the reference dynamics (blue) can be wasteful since we rarely obtain paths (orange) which reach (the vicinity of) the terminal set $\mathcal{B}$.   This is a major challenge for techniques based on importance sampling or Monte Carlo estimation, even when adding a control term to the reference dynamics.   By contrast, our approach optimizes a tractable variational distribution over transition paths with 
     a parameterization which satisfies the initial and terminal conditions by design.
    }
    \label{fig:main-figure}
    \vspace*{-.4cm}
\end{figure}

\paragraph{Doob's $h$-Transform} 
For Brownian motion diffusion processes, conditioning is known to be achieved by Doob's $h$-transform \citep{doob1957conditional, sarkka2019applied}. However, solving this problem amounts to estimating rare event probabilities or matching a complex target distribution. Approaches which involve simulation of trajectories to construct Monte Carlo expectations or importance sampling estimators \citep{papaspiliopoulos2012importance, schauer2017guided, yan2022learning, holdijk2024stochastic} can be extremely inefficient if the target event is rare or endpoint distribution is difficult to match. Recent methods based on score matching \citep{heng2021simulating} or nonlinear Feynman-Kac formula \citep{chopin2023computational} still require simulation during optimization.

\paragraph{Variational Formulation of Doob's $h$-Transform} In this work, we propose a variational formulation of Doob's $h$-transform as the solution to an optimization on the space of paths of probability distributions.    We focus on solving for the Doob transform conditioning on a particular terminal point, which is natural in the TPS setting (see \cref{fig:main-figure}).   Taking inspiration from recent bridge matching methods \citep{peluchetti2021nondenoising, peluchetti2023diffusion, liu2022let, lipman2022flow, shi2023diffusion, liu2023generalized}, we propose a parameterization with the following attractive features.
\begin{enumerate}
\item \textbf{Every Sample Matters.} 
    In contrast to most existing approaches, our training method is \textit{simulation-free}, thereby avoiding computationally wasteful simulation methods to estimate rare-event probabilities and inefficient importance or rejection sampling.   We thus refer to our approach as being \textit{sample-efficient}.    
    \item \textbf{Optimization over Sampling.}
    We propose an expressive variational family of approximations to the conditioned process, which are tractable to sample and can be optimized using neural networks with end-to-end backpropagation.
    \item \textbf{Problem-Informed Parameterization.} 
    Our parameterization enforces the boundary conditions \textit{by design}, thereby reducing the search space for optimization and efficiently making use of the conditioning information.   
\end{enumerate}

We begin by linking the problem of transition path sampling to the Doob's $h$-transform and recalling background results in \cref{sec:background}.  We present our variational formulation in \cref{sec:lagrangian} and detail our optimization algorithm throughout \cref{sec:parameterization}. {We demonstrate the ability of our approach to achieve comparable performance to \gls{MCMC} methods with notably improved efficiency on synthetic, and real-world molecular simulation tasks in \cref{sec:experiments}.}
\newcommand{\ifsmall}{\small}
\secvheader
\secvheader
\section{Background}\label{sec:background}
\secvheader
\subsection{Transition Path Sampling}\label{sec:tps}
\secvheader
Consider a forward or reference stochastic process with states $\tpsx_t$ and the density of transition probability $\prob_{t+dt}(y|\tpsx_t= \tpsx ) \coloneqq \prob(\tpsx_{t+dt} = \tpsy \cond \tpsx_t= \tpsx)$. Starting from an initial point $\tpsx_0 = A$, the probability density of a discrete-time path is given as
\ifsmall 
\begin{align}
    \prob(\tpsx_T,\ldots, \tpsx_{dt}\cond \tpsx_0 = A) &= \prod_{t=dt}^{T-dt} \prob(\tpsx_{t+dt}\cond \tpsx_t) \cdot \prob(\tpsx_{dt}\cond \tpsx_0 = A) .
\end{align}
\normalsize
The problem of rare event sampling aims to condition this reference stochastic process on some event at time $T$, for example, that the final state belongs to a particular set $\tpsx_T \in \setB$.   
We are interested in sampling from the entire \textit{transition path}, namely the 
posterior distribution over intermediate states
\ifsmall 
\begin{align}
    \prob(\tpsx_{T-dt},\ldots, \tpsx_{dt}\cond \tpsx_0 = A, \tpsx_T \in \setB) = \frac{\prob(\tpsx_T \in \setB,\tpsx_{T-dt}\ldots,\tpsx_{dt}\cond \tpsx_0=A)}{\prob(\tpsx_T \in \setB \cond \tpsx_0 = A)} . \label{eq:transition_path}
\end{align}
\normalsize
Moving to continuous time, we focus on the transition path sampling problem in the case where the reference process is given by a Brownian motion.
In particular, we are motivated by applications in computational chemistry \citep{dellago2002transition, weinan2010transition}, where the reference process is given by molecular dynamics following either overdamped Langevin dynamics,
\small
\begin{align}
        \ifbm{d\tpsx_t} &= - ( \gamma \mass)^{-1} \nabla_x U(\ifbm{\tpsx_t}) \cdot dt + 
         (\gamma \mass)^{-\sfrac{1}{2}} 
         \sqrt{2 {k_B \mathcal{T}}} \cdot \ifbm{dW_t}\,, 
    \label{eq:first_order}
\end{align}
\normalsize
{or the second-order Langevin dynamics with spatial coordinates $\tpsxbar_t$ and velocities $\ifbm{\mom_t}$,}
\ifsmall
\begin{align}
   \ifbm{d\tpsxbar_t} ~&= \ifbm{\mom_t} \cdot dt \,,  \qquad  \ifbm{d\mom_t} ~= \Big(-\mass^{-1}  \nabla_x U(\tpsxbar_t) - \gamma \mom_t \Big) \cdot dt + \mass^{-\sfrac{1}{2}} \sqrt{2 \gamma k_B \mathcal{T}} \cdot \ifbm{dW_t} \,. \label{eq:second_order}
\end{align}
\normalsize
for a potential energy function $U$, where $\ifbm{W_t}$ denotes the Wiener process. Note that $k_B \cT$ is the Boltzman constant times temperature, $\mass$ is the mass matrix, and $\gamma$ is the friction coefficient.

\secvheader
\subsection{Doob's \texorpdfstring{$h$}{h}-transform}\label{sec:doobs-h-transform}
\secvheader

Doob's $h$-transform addresses the question of conditioning a reference Brownian motion to satisfy a terminal condition such as $\tpsx_T \in \setB$, thereby providing an avenue to solve the transition path sampling problem described above. Without loss of generality, and to provide a unified treatment of the dynamics in \labelcref{eq:first_order}--\labelcref{eq:second_order}, we consider the forward or reference \gls{SDE},
\begin{align}
  {{\bbPref_{0:T}:}} \quad \qquad \quad \qquad 
    \ifbm{dx_t}= \btofxt \cdot dt + \Xi_t~ \ifbm{dW_t} \,, \;\; \qquad \quad \tpsx_0 \sim \prob_0\,, \qquad \qquad \qquad \label{eq:ref_sde}
\end{align} 
with drift vector field $b_t: \mathbb{R}^{\doobdim} \to \mathbb{R}^{\doobdim}$ and diffusion coefficient matrix $\Xi_t \in \mathbb{R}^{\doobdim \times \doobdim}$ 
such that $\Dt \coloneqq \frac{1}{2} \Xi_t \Xi_t^T$ is positive definite.\footnote{Note that the second-order dynamics in \labelcref{eq:second_order} can be represented using
\small 
\begin{align*}
    \ifbm{\tpsx_t} = \begin{bmatrix}
        \tpsxbar_t \\
        \ifbm{\mom_t}
    \end{bmatrix}\,,\quad\;\;
    \ifbm{b_t}(\tpsx_t) = \begin{bmatrix}
        \ifbm{\mom_t}
        \\
        -\mass^{-1}  \nabla_x U(\tpsxbar_t) - \gamma \mom_t
    \end{bmatrix}\,,\;\;\quad 
    \ifbm{\Dt} = 
    \begin{bmatrix}
        0 & 0\\
        0 & \mass^{-\sfrac{1}{2}} \sqrt{2 \gamma k_B \mathcal{T}}
    \end{bmatrix}\,.  
\end{align*}
\normalsize
}
We denote the induced path measure as $\bbPref_{0:T} \in \mathcal{P}(\mathcal{C}([0,T]\rightarrow \bbR^{\doobdim}))$, i.e. a measure over continuous functions from time to $\bbR^{\doobdim}$.

Remarkably, Doob's $h$-transform \citep[Sec. 7.5]{doob1957conditional,sarkka2019applied} shows that conditioning the reference process \labelcref{eq:ref_sde} on $\tpsx_T \in \setB$ results in another Brownian motion process.

\begin{restatable}{proposition}{jamison}\label{prop:doob_sde}
{\textup{[\citet[Thm. 2]{jamison1975markov}]}}
Let $\hset(\tpsx,t) \coloneqq \prob_T(\tpsx_T \in \setB\cond \tpsx_t= \tpsx)$ denote the conditional transition  probability of the reference process in \labelcref{eq:ref_sde}.  Then,
\begin{align}
 \hspace*{-.2cm} {{\bbP^*_{0:T}:}} \qquad   \ifbm{d\tpsx_{t|T}} &= \Big(
    b_t(x_{t|T}) + 2 \Dt \nabla_{\tpsx} \log \hset(\tpsx_{t|T}, t)
    \Big)\cdot dt 
    + \Xi_t ~ \ifbm{dW_t}\, \qquad x_0 \sim \rho_0 \quad \label{eq:h_sde}
\intertext{
where we use $x_{t|T}$ to denote a conditional process.   The \gls{SDE} in \labelcref{eq:h_sde} is associated with the following transition probabilities for $s < t <T$,}
       &{ \prob_{t}(\tpsy \cond  x_s = \tpsx, \tpsx_T \in \setB ) = \frac{\hset(\tpsy, t)  }{\hset(\tpsx, s)}  \prob_{t}(\tpsy \cond x_s = \tpsx)}
    ,  \label{eq:conditioned}
\end{align}
 Note that all of our subsequent results hold for the case when $\mathcal{B}$ is a point-mass, with the only change being that the $h$-function becomes a density, $\hpt(x,t) = \rho_T(B\cond x_t = x)$.
\end{restatable}
See \cref{app:doob} for proof, and note that \labelcref{eq:conditioned} is simply an application of Bayes rule 
$ \prob_t( \tpsy \cond \tpsx_s = \tpsx, \tpsx_T \in \setB ) =  \prob_T(x_T \in \setB | x_t = y) \prob_t( \tpsy \cond \tpsx_s = \tpsx)  /  \prob_T(x_T \in \setB | x_s = x)$ with the unconditioned or reference transition probability as the prior. Furthermore, the conditioned transition probabilities in \labelcref{eq:conditioned} allow us to directly construct the transition path \labelcref{eq:transition_path}. Using Bayes rule, we have 
\ifsmall
\begin{align*}
    \prob(\tpsx_{T-dt},\ldots, \tpsx_{dt}\cond \tpsx_0 = A, \tpsx_T \in \setB) 
 = \frac{\hset(\tpsx_{T-dt},T-dt)}{\hset(A,0)} { \prob(\tpsx_{T-dt}\ldots,\tpsx_{dt}\cond \tpsx_0=A)}
\end{align*}
\normalsize
after telescoping cancellation of $h$-functions and rewriting the denominator in \labelcref{eq:transition_path} as $\hset(A,0)$
Thus, we can solve the \gls{TPS} problem by exactly solving for the $h$-function and simulating the \gls{SDE} in \labelcref{eq:h_sde}.

\newcommand{\pcondt}{{t|0,T}}

Finally, the $h$-process and temporal marginals $\rho_t(x | x_0 = A, x_T \in \setB)$ of the conditioned process satisfy the following forward and backward Kolmogorov equations, which will be useful in deriving our variational objectives in the next section.   Note, we use $\langle \nabla_{\tpsx}, \sbullet \rangle = \text{div}(\sbullet)$ for the divergence operator, and we use $\rho_{t|0,T}$ to indicate the dependence on both $x_0 = A$ (via the initial condition of \cref{eq:doob_fp}) and $x_T \in \setB$ (via the $h$-transform $h_\setB$).   
See \cref{app:doob} for the proof.
\begin{restatable}{proposition}{doobpdes}\label{prop:doob_pdes}
    The following PDEs are obeyed by (a) the 
    density of the conditioned process
    {$\condp_\pcondt(\tpsx) \coloneqq \prob_t(x \cond x_0 = A,  \tpsx_T \in \setB)$}
    and (b) the $h$-function $\hset(\tpsx,t)$,
\begin{subequations}
    \begin{align}
    &\resizebox{.935\textwidth}{!}{\ensuremath{\dderiv{\condp_{\pcondt}(\tpsx)}{t} + \inner{\nabla_x}{\condp_\pcondt(\tpsx) \big(b_t(\tpsx) + 2 \Dt\nabla_x \log \hset(\tpsx, t)\big)} - \mathlarger{\sum}\limits_{ij} (\Dt)_{ij} \dderiv{^2}{x_i\partial x_j}\condp_\pcondt(\tpsx) = 0\,,} } \label{eq:doob_fp} \\
    &\deriv{\hset(\tpsx, t)}{t} + \inner{\nabla_x \hset(\tpsx, t)}{\ifbm{b_{t}}(\tpsx)} + \sum_{ij} (\Dt)_{ij} \deriv{^2}{x_i\partial x_j}\hset(\tpsx, t) = 0\,. \label{eq:doob_hjb}
\end{align}

Reparameterizing \labelcref{eq:doob_hjb} in terms of $s_B(\tpsx, t) \coloneqq \log \hset(\tpsx,t)$, we can also write
\begin{equation}
\resizebox{.93\textwidth}{!}{\ensuremath{\dfrac{\partial s_\setB(\tpsx, t)}{\partial t} + \inner{\nabla s_\setB(\tpsx, t)}{\Dt \nabla s_\setB(\tpsx, t)} + 
\inner{\nabla s_\setB(\tpsx, t)}{b_t(x)}
+ \sum \limits_{ij} (\Dt)_{ij} \dfrac{\partial{^2}}{\partial x_i\partial x_j}s_\setB(\tpsx, t) = 0 .
}}
\label{eq:doob_shjb}
\end{equation}
\end{subequations}
\end{restatable}
\secvheader
\secvheader
\section{Method}\label{sec:method}
\secvheader
We first present a novel variational objective whose minimum corresponds to the Doob $h$-transform in \cref{sec:lagrangian}, and then propose an efficient parameterization to solve for the $h$-transform in \cref{sec:parameterization}.

\secvheader
\subsection{Doob's Lagrangian}\label{sec:lagrangian}
\secvheader

Consider reference dynamics given in the form of either \labelcref{eq:first_order} or \labelcref{eq:second_order}, with known drift $\ifbm{b_t}$ or energy $U$.   We will restrict our attention to conditioning on a terminal rare event of reaching a given endpoint $\ifbm{x_T} = B$, along with an initial point $\ifbm{x_0}= A$. We approach solving for Doob's $h$-transform via a \textit{least action principle} where, in the following theorem, 
we define a Lagrangian action whose minimization yields the optimal $q_\condt^*(x) = \condp_\condt(x)$ and $v_\condt^*(x) = \nabla_{\ifbm{x}} \log \hpt(\ifbm{x},t)$ from \cref{prop:doob_sde} and \ref{prop:doob_pdes}.

\begin{mytheorem}
\label{th:var_doob}
The following Lagrangian action functional has a unique solution which 
matches the Doob $h$-transform in \cref{prop:doob_pdes},
\begin{subequations}
\label{eq:doob_lagrangian_all}
\hspace*{-.31cm} 
\begin{align}
   &\mathcal{S} =~ \min_{q_\condt, v_\condt} \int_0^T dt\;\int dx\; q_\condt(x) \inner{v_\condt(x)}{\Dt ~  v_\condt(x)}\,, \label{eq:doob_lagrangian} \\[1.25ex]
   &\resizebox{.91\textwidth}{!}{\text{s.t.}~ \ensuremath{
   \dderiv{q_\condt(x)}{t} = - \inner{\nabla_x}{q_\condt(x)\left(b_t(x) + 2\Dt~ v_\condt(x)\right)} + \mathlarger{\sum}\limits_{ij} (\Dt)_{ij} \dderiv{^2}{x_i\partial x_j}q_\condt(x), 
   }}
   \label{eq:doob_lagrangian_fp} \\
    &~\phantom{\deriv{q_\condt}{t}\hspace{.016\textwidth}}  q_0(x) = \delta(x-A), \qquad q_T(x) = \delta(x-B)\,.\label{eq:doob_lagrangian_fp_boundary}
\end{align}
\end{subequations}
The optimal $q_\condt^*(x)$ obeys \labelcref{eq:doob_fp}, and $v_\condt^*(x) = \nabla_{\ifbm{x}} \log \hpt(\ifbm{x},t) = \nabla_{\ifbm{x}} s_B(\ifbm{x},t)$ obeys \labelcref{eq:doob_hjb}-\labelcref{eq:doob_shjb}.
\end{mytheorem}

This objective will form the basis for our computational approach, with proof of \cref{th:var_doob} deferred to \cref{app:doobs_lagrangian}.  We proceed briefly to contextualize
our variational objective and highlight several optimization challenges which will be solved by our proposed parameterization in \cref{sec:parameterization}.

\paraorsubsec{Unconstrained Dual Objective}
Introducing Lagrange multipliers to enforce the constraints in \labelcref{eq:doob_lagrangian_fp}--\labelcref{eq:doob_lagrangian_fp_boundary} and eliminating $v_\condt$, we obtain an alternative, unconstrained version of \labelcref{eq:doob_lagrangian}.
\newcommand{\condzero}{_0}
\newcommand{\condend}{_T}
\newcommand{\condst}{}

\begin{restatable}{corollary}{minmax}\label{th:minmax}
    The Lagrangian objective in \cref{th:var_doob} which solves Doob's $h$-transform is equivalent to
    \small 
    \begin{equation}
        \resizebox{\textwidth}{!}{\ensuremath{
                  \cS= \min\limits_{q_\condt}\max\limits_{s} 
        ~ s_B(B, T) - s_B(A, 0)
        - \mathlarger{\int_0^T} dt\; \mathlarger{\int} dx ~ q_\condt \left( \dderiv{s_B}{t} + \inner{\nabla
         s_B}{\Dt \nabla
         s_B} + \inner{\nabla
         s_B}{b_t} + \inner{\nabla
         }{\Dt \nabla s_B}\right)\,  }} \nonumber 
    \end{equation}
    \normalsize    
    if $q_\condt$ satisfies \labelcref{eq:doob_lagrangian_fp_boundary}.
    Note $v_\condt(x) = \nabla_{x} s_B(x,t)$, with $s^*_B(x,t) = \log \hpt(x,t)$ at optimality.
    \footnote{
    Compared to \labelcref{eq:doob_shjb},
    we write $\sum_{ij} (\Dt)_{ij} \frac{\partial{^2}}{\partial x_i\partial x_j}s_\setB(\tpsx, t)  = \langle \nabla ,  \Dt \nabla s_\setB(\tpsx,t) \rangle$ for 
    simplicity of notation.
    }
\end{restatable}
This objective is similar to the objectives optimized by Action Matching methods \citep{neklyudov2023action,neklyudov2023computational}.  Notably, the objective in \cref{th:minmax} is expressed \textit{directly} in terms of the (log) of the $h$-function for fixed conditioning information $x_T = B$.   We also note that the Hamilton Jacobi-style quantity, whose expectation appears in the final term, is zero at optimality in \labelcref{eq:doob_shjb} of \cref{prop:doob_pdes}.

\paraorsubsec{Path Measure Perspective}
We next relate our variational objective in \cref{th:var_doob} to a KL divergence optimization over path measures. Let $\bbPref_{0:T}$ denote the law of the reference \gls{SDE} in \labelcref{eq:ref_sde} with fixed $\bbPref_0 = \delta(x_0 - A)$.  Let $\bbQ^{v}_{0:T}$ denote the law of a controlled process similar to \labelcref{eq:h_sde}, but with a variational $v_\condt$ in place of $\nabla_{\ifbm{x}} \log \hset$,
\begin{align}
\bbQ^{v}_{0:T}:  \qquad 
\ifbm{d\tpsx_t} = \big(
    b_t(x_{t|T}) + 
    2 \Dt ~  v_\condt(x_{t|T})
    \big)\cdot dt 
    + \Xi_t ~ \ifbm{dW_t}\,, \qquad  x_0 = A . \label{eq:hv_sde}
\end{align}
Note that the density
$q_\condt$ of $\bbQ^{v}_{0:T}$ evolve according to the Fokker-Planck equation in \labelcref{eq:doob_lagrangian_fp} \citep[Sec. 5.2]{sarkka2019applied} . Using the Girsanov Theorem, the objective in \labelcref{eq:doob_lagrangian} can then be viewed as a KL divergence minimization over path measures $\bbQ^v_{0:T}$ which satisfy the boundary constraints.
\begin{restatable}{corollary}{sb}\label{th:sb}
The following \gls{SB} problem
\begin{align}
 \S \coloneqq \min \limits_{\bbQ^v_{0:T} ~\text{s.t.} ~ \bbQ^v_0 = \delta_A, \bbQ^v_T = \delta_B} D_{KL}[ \bbQ^v_{0:T} : \bbPref_{0:T}] \label{eq:sb}
\end{align}
yields the path measure $\bbP^*_{0:T}$ associated with the \gls{SDE} in \labelcref{eq:h_sde} as its unique minimizing argument.  
The temporal marginals of $\bbP^*_{0:T}$ are equal to those which optimize the Lagrangian objective in \cref{th:var_doob}.
\end{restatable}
Our Lagrangian action minimization thus {corresponds to the solution of an} \gls{SB} problem \citep{schrodinger1932theorie, leonard2014survey} with Dirac delta functions as the endpoint measures.   Our objective in \labelcref{eq:doob_lagrangian} particularly resembles optimal control formulations of \gls{SB} \citep[Prob. 4.4, 5.3]{chen2016relation,chen2021stochastic}. While it is well-known that the Doob $h$-transform (and large deviation theory more generally) plays a role in the solution to \gls{SB} problems \citep{jamison1975markov, leonard2014survey}, our interest in the transition path sampling problem leads to specific computational decisions below.  See \cref{sec:related} for further discussion.

\begin{mymath}
\paraorsubsec{Challenges of Optimizing \labelcref{eq:doob_lagrangian}} 
\textup{ We highlight several distinctive features of our problem which inform the development of new computational methods in \cref{sec:parameterization}.}
\begin{myenum}
    \item  First, we perform optimization over the \textit{first} argument of the KL divergence in \labelcref{eq:sb}, indicating that we need to be able to efficiently sample from the conditioned process in \labelcref{eq:hv_sde} or $q_\condt$ in \labelcref{eq:doob_lagrangian_all}.  This appears challenging due to the nonlinearity of both the reference and variational drifts,  $\ifbm{b_t}$ and $v_\condt$.\label{challenge1}
\item For a given $q_{\condt}$, it can be difficult to solve for $v_\condt$ which satisfies the Fokker-Planck equation in \labelcref{eq:doob_lagrangian_fp} or $\nabla s$ which solves the inner optimization in \cref{th:minmax}. \label{challenge2}
\item Finally, we would like to strictly enforce the boundary constraints on $q_\condt$ or $\bbQ_{0:T}^v$ to avoid 
simulating or wasting computation on
trajectories for which $x_T \neq B$. \label{challenge3}
\end{myenum}
\end{mymath}

In fact, our parameterization of $q_\condt$ in \cref{sec:parameterization} will \textit{completely avoid} simulation of the SDE in \labelcref{eq:hv_sde} during training (\cref{challenge1}), provide { \textit{analytic} solutions
for $v_\condt$ satisfying \labelcref{eq:doob_lagrangian_fp} with a given $q_\condt$ (\cref{challenge2}), and \textit{exactly} enforce the boundary constraints (\cref{challenge3}).

\secvheader
\subsection{Computational Approach}\label{sec:parameterization}
\secvheader

We now propose a family of Gaussian (mixture) path parameterizations $q_\condt$ which overcome the computational challenges posed in the previous section, while still maintaining expressivity.   
We present all aspects of our proposed method in the context of the first-order dynamics \labelcref{eq:first_order} in \cref{sec:first_order}, before presenting extensions to mixture paths and the second-order setting \labelcref{eq:second_order} in \cref{sec:second_order}--\ref{sec:mixture}.

\secvheader
\subsubsection{First-Order Dynamics and General Approach}  \label{sec:first_order}

\paragraph{Tractable Drift $v_\condt$ for Variational Doob Objective}
We begin by considering a modification of the Fokker-Planck constraint in \labelcref{eq:doob_lagrangian_fp}, 
with all drift terms absorbed into a single vector field $u_\condt$,
\ifsmall
\begin{align}
    \deriv{q_\condt(x)}{t} = - \inner{\nabla_x}{q_\condt(x) ~ u_\condt(x)} + \sum_{ij} (\Dt)_{ij} \deriv{^2}{x_i\partial x_j}q_\condt(x) . \label{eq:fpe_general}
\end{align} 
\normalsize
For arbitrary $q_\condt$, solving for \textit{any} $u_\condt(x)$ satisfying \labelcref{eq:fpe_general} can be a difficult optimization problem, whose solution is not unique without some cost-minimizing assumption \citep{neklyudov2023action}.

To sidestep this optimization, and address \cref{challenge2}, we restrict attention to variational families of $q_\condt \in \mathcal{Q}$ where it is \textit{analytically tractable} to calculate a vector field $\uq$ which satisfies \labelcref{eq:fpe_general}. We first consider the family of Gaussian paths $\mathcal{Q}_G$, in similar fashion to (conditional) flow matching methods \citep{lipman2022flow, tong2023conditional, liu2023generalized}, with proof in \cref{app:gaussians}. 
\begin{restatable}{proposition}{gaussianpath}\label{prop:gaussian_path}
For the family of endpoint-conditioned marginals $q_\condt(x) = \Normal(x\cond \mu_\condt, \Sigma_\condt)$, 
\begin{align}
\uq(x) &\coloneqq \deriv{\mu_\condt}{t} + \left[\frac{1}{2}\deriv{\Sigma_\condt}{t}\Sigma_\condt^{-1} - \Dt ~ \Sigma_\condt^{-1}\right]\big( x-\mu_\condt \big) \label{eq:uq}   \quad
\end{align}
\normalsize
satisfies the Fokker-Planck equation \labelcref{eq:fpe_general} for $q_\condt$ and diffusion coefficients $\Dt = \frac{1}{2} \Xi_t \Xi_t^T$.
\end{restatable}
\vspace*{-.3cm}
Given $\uq$ corresponding to $q_\condt$, we can simply solve for the $v_\condt$ satisfying the Fokker-Planck equation in \labelcref{eq:doob_lagrangian_fp} in  our variational Doob objective (\cref{th:var_doob}). Since $\Dt$ was assumed to be invertible and the base drift $\ifbm{b_t}$ is known, we have
\begin{align}
    \vq(x) = \frac{1}{2} \left(\Dt\right)^{-1}\left( \uq(x) - \ifbm{b_t}(x) \right). \label{eq:vstar}
\end{align}
\normalsize
We may now evaluate terms involving $v_\condt$ in our Lagrangian objective in \labelcref{eq:doob_lagrangian_all} using \labelcref{eq:vstar} directly, without spending effort to solve an inner minimization over $v_\condt$ (thus addressing \cref{challenge2}).    

\begin{figure*}
\begin{minipage}{.57\textwidth}
\vspace*{-.6cm}
\begin{algorithm}[H]
    \caption{Training (Single Gaussian)}
  \begin{algorithmic}
    \STATE \textbf{Input}:  Reference drift $b_t$, diffusion matrix $\Dt$
    \STATE \phantom{\textbf{Input}:} Conditioning endpoints 
    \WHILE{not converged}
    \vspace*{.05cm}
    \STATE Sample $t \sim \cU (0, T)$
    \STATE Sample $x_t \sim q^{(\theta)}_{t|0,T}$ using \labelcref{eq:gaussian_firstorder}
    \STATE Calculate $\uq(x_t)$ using \labelcref{eq:uq}
    \STATE Calculate $\vq(x_t)$ using $\uq(x_t)$, $\ifbm{b_t}(x_t)$, \labelcref{eq:vstar} 
    \STATE Calculate $\cL = \langle \vq(x_t), \Dt ~ \vq(x_t) \rangle$ (\cref{th:var_doob})
    \STATE Update $\theta \gets \text{optimizer}(\theta, \nabla_{\theta}\cL)$
    \ENDWHILE
    \RETURN{$\theta$}
  \end{algorithmic}
  \label{alg:training}
  \end{algorithm}%
  \end{minipage}\hspace*{.02\textwidth}
  \begin{minipage}{.41\textwidth}
  \vspace*{-.6cm}
  \begin{algorithm}[H]
    \caption{Sampling Trajectories}
  \begin{algorithmic}
   \STATE \textbf{def} \texttt{get\_drift}($x_t$, $t$):
   \STATE \phantom{\textbf{def}} Evaluate $\mu_\condt^{(\theta)}$, $\Sigma_\condt^{(\theta)}$ at $t$%
   \STATE \phantom{\textbf{def}} \textbf{return} drift $\uq(x_t)$ using \labelcref{eq:uq}
\vspace*{.1cm}
      \STATE Sample initial state $x_0 \sim \cN(A, \sigma^2_{\text{min}})$ 
   \RETURN{SDESolve($x_0$, \texttt{get\_drift}, $T$)}
    \end{algorithmic}
  \label{alg:sampling}
  \end{algorithm}%
  \caption*{\normalsize Algorithms for training with a single Gaussian path (\cref{alg:training}) and sampling or generating transition paths at test time (\cref{alg:sampling}). Note that we sample from the marginals $q_\condt$ during training, but generate paths by simulating the \gls{SDE} \labelcref{eq:hv_sde}.}
    \end{minipage}
    \secvheader
    \vspace*{-.4cm}
  \end{figure*}

\paragraph{Optimization over $q_\condt$ satisfying Boundary Constraints}  
Given the ability to evaluate $\vq$ for a given $q_\condt \in \mathcal{Q}_{G}$ as above, our variational Doob objective in \labelcref{eq:doob_lagrangian} reduces to a single optimization over the marginals $q_\condt$ of a conditioned process which satisfies the boundary conditions \labelcref{eq:doob_lagrangian_fp_boundary}.

We consider parameterizing the mean $\mu_\condt$ and covariance $\Sigma_\condt$ of our Gaussian path $q_\condt$ using a neural network.  For simplicity, we consider a diagonal parameterization $\Sigma_\condt = \texttt{diag}(\{\sigma_{\condt, d}^2\}_{d=1}^D)$. We parameterize a neural network $\textsc{nnet}_\theta: [0,T] \times \xdomain \times \xdomain \rightarrow  \mathbb{R}^{\xdim}   \times \mathbb{R}^{\xdim} $ which inputs time $t$ and boundary conditions $x_0= A, x_T=B$, and outputs vectors of mean perturbations and per-dimension variances.   Finally, using index notation to separate the output, we construct 
\vspace*{-.4cm}
\begin{subequations}
\begin{align}
x_\condt &= \mu_\condt^{(\theta)} + \Sigma_\condt^{(\theta)} ~ \epsilon, \quad \text{where} \quad  \epsilon \sim \cN(0,\mathbb{I}_D) . \label{eq:x_firstorder} \\ 
& ~~\phantom{\Sigma} \mu_\condt^{(\theta)}
= \left(1-  \frac{t}{T} \right) A  + \frac{t}{T}~ B + \frac{t}{T} \left(1- \frac{t}{T} \right) \textsc{nnet}_\theta(t, A, B)_{[:\xdim]} 
\label{eq:mu_firstorder}
\\
& ~~\phantom{\mu} \Sigma_\condt^{(\theta)}
=
\frac{t}{T} \left(1- \frac{t}{T} \right)
\texttt{diag}\left( \textsc{nnet}_\theta(t, A, B)_{[\xdim:]} \right) 
+ \sigma^2_{\text{min}} \bbI_\xdim
. \label{eq:sigma_firstorder}  
\end{align}
\label{eq:gaussian_firstorder}
\end{subequations}
\normalsize
Crucially, our Gaussian parameterization addresses \cref{challenge1}, in that we can easily draw samples $x_\condt \sim q_\condt$ from our variational conditioned process \labelcref{eq:doob_lagrangian_fp} \textit{without simulating} the corresponding \gls{SDE} with nonlinear drift \labelcref{eq:hv_sde}.  
Further, the coefficients in \labelcref{eq:mu_firstorder,eq:sigma_firstorder} ensure that, as $t\rightarrow 0$ or $t\rightarrow T$, our parameterization satisfies the (smoothed) boundary conditions by design (\cref{challenge3}). Although we add $\sigma_{\text{min}}^2$ to ensure invertibilty of $\Sigma_\condt$ (see \labelcref{eq:uq}), we preserve $q_0(x_0) = \cN(x_0 \cond A, \sigma_{\text{min}}^2 \bbI_\xdim) \approx \delta(x_0 - A)$ and $ q_T(x_T) = \cN(x_T \cond B, \sigma_{\text{min}}^2 \bbI_\xdim) \approx \delta(x_T - B)$.

\paragraph{Reparameterization Gradients}
Having shown that our parameterization satisfies the constraints \labelcref{eq:doob_lagrangian_fp}-\labelcref{eq:doob_lagrangian_fp_boundary} by design,  we can finally optimize our variational Doob objective with respect to $q_\condt \in \cQ_G$ using the reparameterization trick \citep{kingma2013auto, rezende2014stochastic}. In particular, for the expectation at each $t$ in \labelcref{eq:doob_lagrangian}, we rewrite
\ifsmall
\begin{align}
\resizebox{\textwidth}{!}{\ensuremath{
    \nabla_{\theta} \mathbb{E}_{q_\condt^{(\theta)}(x)}\left[ \inner{\vqtheta(x)}{\Dt ~ \vqtheta(x)}\right] = \mathbb{E}_{\cN(\epsilon | 0, \bbI_\xdim)}\left[ \nabla_\theta \inner{\vqtheta\big(g(t,\epsilon;\theta)\big)}{\Dt ~ \vqtheta\big(g(t,\epsilon;\theta)\big)} \right]
    }} ,
   \nonumber
\end{align}
\normalsize
where $x = g(t,\epsilon;\theta)$ is the mapping in \labelcref{eq:gaussian_firstorder} and $\vqtheta$ depends on $\theta$ via $\mu_\condt^{(\theta)}$, $\Sigma_\condt^{(\theta)}$ in \labelcref{eq:uq}--\labelcref{eq:vstar}. 

\paragraph{Full Training Algorithm}
In practice, we sample a batch of times $\{t_i\}_{i=1}^M$ uniformly from the interval $t \in [0,T]$. For each time point, we approximate the gradient using a single-sample estimate of the expectation above (or \labelcref{eq:doob_lagrangian_all}), which yields a simulation-free training procedure. The full training algorithm is outlined in \cref{alg:training}.  

\paragraph{Sampling of Trajectories}  While we sample directly from $q^{(\theta)}_{t|0,T}$ during training, we can sample full trajectories which obey this sequence of marginals at test time (\cref{alg:sampling}). In particular, we simulate SDE trajectories with drift $\uq(x)$ and diffusion coefficient $\Dt$ using an appropriate solver. 
Note that this generation scheme sidesteps computationally expensive evaluation of the force field or base drift $\ifbm{b_t}(x_t)$.   We visualize example sampling trajectories in  \cref{fig:mueller-path-histograms}. 

\secvheader
\subsubsection{Second-Order Dynamics}
\secvheader
\label{sec:second_order}
 To handle the case of the second-order dynamics in \labelcref{eq:second_order},  we can adapt our recipe from the previous section with minimal modifications by extending the state space $\tpsx \in \xdomain$ to include velocities $\mom$, with $\tpsx = ( \tpsxbar , \mom) \in \mathbb{R}^{2\xdim}$.
However, note that the dynamics in \labelcref{eq:second_order} are no longer stochastic in the spatial coordinates $\tpsxbar$.  
To ensure invertibility of $\Dt$ and existence of the $h$-transform, we add a small nonzero diffusion coefficient in the coordinate space $\tpsxbar$, so that the reference process in \cref{eq:ref_sde} is given by
\small 
\begin{align}
   \hspace*{-.22cm} \ifbm{\tpsx_t} = \begin{bmatrix}
        \tpsxbar_t \\
        \ifbm{\mom_t}
    \end{bmatrix},\quad
    \ifbm{b_t}(\tpsx_t) = \begin{bmatrix}
        \ifbm{\mom_t}
        \\
         -\mass^{-1}  \nabla_x U(\tpsxbar_t) - \gamma \mom_t
    \end{bmatrix},\quad 
    \ifbm{\Xi_t} = 
    \begin{bmatrix}
        \xi_{\text{min}}\bbI_\xdim  & 0\\
         0 & \mass^{-\sfrac{1}{2}} \sqrt{2 \gamma k_B \mathcal{T}}
    \end{bmatrix}.  
    \label{eq:ref_second_order}
\end{align}
\normalsize
  All steps in our algorithm proceed in similar fashion to \cref{sec:first_order}.
We now parameterize 
$q_\condt(\tpsxbar , \mom) $ using 
$\textsc{nnet}_\theta: [0,T] \times \mathbb{R}^{2\xdim}  \times \mathbb{R}^{2\xdim}  \rightarrow  \mathbb{R}^{2\xdim}   \times \mathbb{R}^{2\xdim} $, which outputs mean perturbations and per-dimension variances to calculate $\mu_\condt^{\tpsxbar}, \mu_\condt^{\mom}$ and $\Sigma_\condt^{\tpsxbar}, \Sigma_\condt^{\mom}$ and sample $(\tpsxbar, \mom)$, as in \labelcref{eq:gaussian_firstorder}.   
{Note that we parameterize $\Sigma_\condt^{\tpsxbar}, \Sigma_\condt^{\mom}$ separately, matching the block diagonal form} of \labelcref{eq:ref_second_order}.
We calculate $v_\condt^{(q)}(\tpsxbar,\mom) \coloneqq [v_\condt^{\tpsxbar(q)},
        v_\condt^{\mom(q)}]$ from $u_\condt^{(q)}(\tpsxbar,\mom) = [ u_\condt^{\tpsxbar(q)},
        u_\condt^{\mom(q)}]$ as in \labelcref{eq:uq}--\labelcref{eq:vstar}, 
        with $\Dt^{-1} =(\frac{1}{2}\Xi_t \Xi_t^T)^{-1}$ given by \labelcref{eq:ref_second_order}.
The Lagrangian objective in \labelcref{eq:doob_lagrangian_all} minimizes the norm of the concatenated vector $v_\condt^{(q)}(\tpsxbar,\mom)$, which depends on the reference drift $b_t(\tpsxbar,\mom)$ in \labelcref{eq:ref_second_order}.

\secvheader
\subsubsection{Gaussian Mixture Paths}\label{sec:mixture}
\secvheader

Note that the true Doob $h$-transform may not yield marginal distributions which are unimodal Gaussians as in the previous section. To increase the expressivity of our variational family of conditioned processes, we now extend our parameterization to mixtures of Gaussians, $q_\condt \in \cQ_{\text{MoG}}^K$.   Given a set of $K$ mixture weights $\bm{w} \coloneqq \{w^k\}_{k=1}^K$ and component Gaussian paths $\{q_\condt^{k}\}_{k=1}^K$, the following identity allows us to obtain the drift $\uq$ of the corresponding mixture distribution $q_\condt$.   

\begin{restatable}{proposition}{mixture}\label{prop:mixture}
Given a set of processes $q_\condt^k(x)$ and mixtures weights $w^k$, the vector field satisfying the Fokker-Planck equation in \labelcref{eq:fpe_general} for the mixture $q_\condt(x) = \sum_k w^k q_\condt^k(x)$ is given by
\small 
\begin{align}
\uq(x) = \sum \limits_{k=1}^K \frac{w^k q_\condt^k(x)}{\sum_{j=1}^K w^j q_\condt^j(x)} u_\condt^{(q,k)}(x)\,, \label{eq:uq_mixture}
\end{align}
\normalsize
where $u_\condt^{(q,k)}(x)$ satisfies the Fokker-Planck equation in \labelcref{eq:fpe_general} for $q_\condt^{k}(x)$.
This identity holds for both first-order dynamics in spatial coordinates only or second-order dynamics in $x = (\tpsxbar, \mom)$.
\end{restatable}
\normalsize
Finally, we can calculate $\vq(x)$ by comparing $\uq(x)$ for the mixture of Gaussian path $q_\condt \in \cQ_{\text{MoG}}^K$ to the reference drift $b_t(x)$ as in \labelcref{eq:vstar}, and proceed to minimize its norm as in \cref{eq:doob_lagrangian_all}. In practice, we use Gumbel softmax reparamerization gradients \citep{maddison2016concrete, jang2017categorical} to optimize the mixture weights $\{w^k\}_{k=1}^K$ alongside the neural network parameters $\{\theta^k\}_{k=1}^K$ for each Gaussian component $\{\mu_{\condt}^{(\theta^k)}, \Sigma_{\condt}^{(\theta^k)} \}_{k=1}^K$ and either first- or second-order dynamics.

\newcommand{\suborpara}[1]{\paragraph{#1}}

\secvheader
\section{Related Work}\label{sec:related}
\secvheader

\suborpara{(Aligned) Schrödinger Bridge Matching Methods}
Many existing `bridge matching' approaches \citep{shi2023diffusion, peluchetti2021nondenoising, peluchetti2023diffusion, liu2022let, lipman2022flow, liu2023i2i} for \gls{SB} and generative modeling rely on convenient properties of Brownian bridges and would require calculating $h$-transforms to simulate bridges for general reference processes. Our conditional Gaussian path parameterization is similar to \citet{liu2023generalized, neklyudov2023computational}, where analytic bridges are not available for \gls{SB} problems with nonlinear reference drift or general costs.

\citet{somnath2023aligned,liu2023i2i} attempt to solve the \gls{SB} problem given access to aligned data $x_0, x_T \sim q^{\text{data}}_{0,T}$ assumed to be drawn from an optimal coupling. While the method in \citet{somnath2023aligned} involves approximating an $h$-transform, their goal is to obtain an unconditioned vector field $v_t$ to simulate a Markov process.   However, \citet{de2023augmented} use Doob's $h$-transform to argue the learned Markov process will not preserve the empirical coupling unless $q^{\text{data}}_{0,T}$ is the optimal coupling for the \gls{SB} problem, and show that an `augmented' $v_{0,t}$ which conditions on $x_0$ can correct this issue.   

After training on a dataset of $x_0, x_T \sim q^{\text{data}}_{0,T}$ pairs using our method, we 
could consider using an (augmented) bridge matching objective \citep{shi2023diffusion, de2023augmented} to distill our learned $v_\condt^{(q)}$ into a vector field $v_{t}$ or $v_{0,t}$ which does not condition on the endpoint.  Our use of a Gaussian path parameterization with samples from a fixed endpoint coupling and no Markovization step corresponds to a simplified 
version of the conditional optimal control step in \citet{liu2023generalized}.

\suborpara{Transition Path Sampling}
We refer to the surveys of \citet{dellago2002transition, weinan2010transition, bolhuis2021transition} for an overview of the \gls{TPS} problem. Least action principles for \gls{TPS} have a long history, building upon the Freidlin-Wentzell \citep{freidlin1998random} and Onsager-Machlup \citep{onsager1953fluctuations, durr1978onsager} Lagrangian functionals in the zero-noise limit and finite-noise cases. In particular, the Onsager-Machlup functional relates maximum a posteriori estimators or `most probable (conditioned) paths' to the minimizers of an action functional similar to \cref{th:var_doob}, where example algorithms include \citep{vanden2008geometric, sheppard2008optimization}. By contrast, our approach targets the \emph{entire} posterior over transition paths using an expressive variational family. While \citet{lu2017gaussian} provide analysis for the Gaussian family, we draw connections with Doob's $h$-transform and extend to mixtures of Gaussians.

Shooting methods are among the most popular for sampling the posterior of transition paths. From a path that satisfies the boundary conditions (obtained, e.g., using high-temperature simulations), shooting picks points and directions to propose alterations, then simulates new trajectories and accepts or rejects using \gls{MH} \citep{juraszek2008rate, borrero2016avoid, jung2017transition, falkner2023conditioning, jung2023machine}. While the \gls{MCMC} corrections yield theoretical guarantees, shooting methods involve expensive \gls{MD} simulations and need to balance high rejection rates with large changes in trajectories. One-way shooting methods sample paths efficiently but yield highly correlated samples. Two-way shooting methods, which we compare to in \cref{sec:experiments}, are more expensive but typically sample diverse paths faster.  Recent machine learning approaches (e.g. \citet{plainer2023transition, lelivre2023}) aim to 
reduce the need for \gls{MD}.  

Finally, various related methods rely on iterative simulation of \gls{SDE} in \labelcref{eq:hv_sde} during training to learn the control drift term. \citet{yan2022learning,holdijk2024stochastic} are motivated from the perspective of stochastic optimal control, while \citet{das2021reinforcement, rose2021reinforcement} develop actor-critic methods using closely-related ideas from soft reinforcement learning.   
The variational method in \citet{das2019variational} optimizes the rate function quantifying the probability of the rare events, while \citet{singh2023variational} solves the Kolmogorov backward equation to learn the Doob's $h$-transform. 
However, all of these methods may be inefficient if the desired terminal state is sampled infrequently.

\secvheader
\secvheader 
\vspace*{-.05cm}
\section{Experiments}\label{sec:experiments}
\secvheader
\secvheader

We investigate the capabilities of our approach across a variety of different settings. We first illustrate features of our method on toy potentials before continuing to real-world molecular systems, including a commonly-used benchmark system, alanine dipeptide, and a small protein, Chignolin. The code behind our method is available at \ourCode{}. Before diving into the experiments, we introduce the evaluation procedure and baseline methods.

\paragraph{Evaluation metrics} In our evaluation, we emphasize two key quantities: accuracy and efficiency. Efficiency is evaluated by the number of calls to the potential energy function, which requires extensive computation and dominates the runtime of larger molecules. For accuracy, we evaluate the log-likelihood of each sampled path and the maximum energy point (saddle point/transition state) along each sampled path. A good method samples many probable paths (i.e., high log-likelihood) and an accurate transition state (i.e., small maximum energy). See \Cref{appendix:exp_details} for further details.

\paragraph{Baselines} We compare our approach against the \gls{MCMC}-based two-way shooting method with uniform point selection with variable or fixed length trajectories. 
We found that two-way shooting produced the most diverse path ensembles among possible baselines, although the acceptance probability can be relatively low for systems dominated by diffusive dynamics~\citep{brotzakis2016onewayshooting} and might be improved by better shooting point selection. 
This baseline gives theoretical guarantees about the ensemble and thus can be considered as a proxy for the ground truth. In that sense, our goal is not to beat two-way shooting but to approximate it with fewer potential evaluations. 

\begin{figure*}[t]
\vspace{-.3cm}
\begin{minipage}{.48\textwidth}
    \centering
    \begin{subfigure}[t]{0.47\textwidth}
        \includegraphics[width=0.9\linewidth]{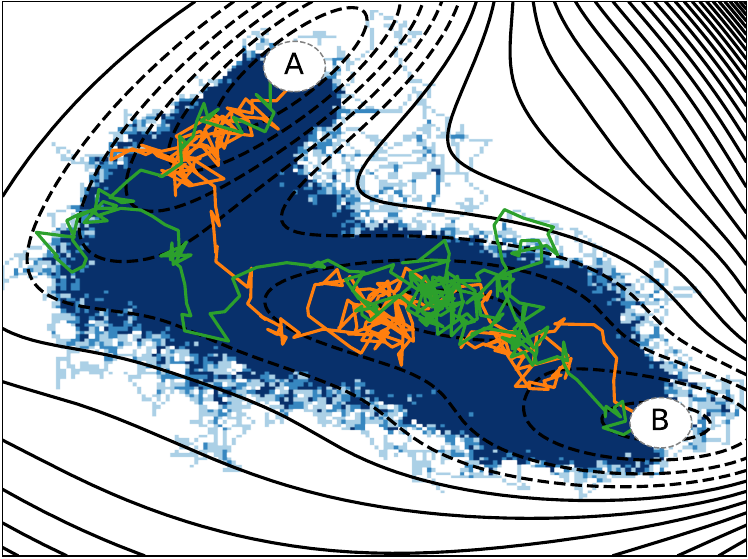}
        \caption{MCMC}
    \end{subfigure}
    \begin{subfigure}[t]{0.47\textwidth}
        \includegraphics[width=0.9\linewidth]{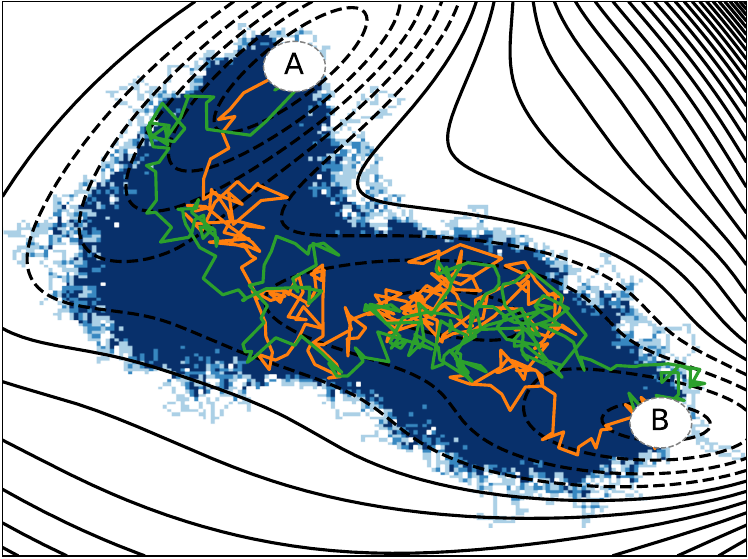}
        \caption{Ours}
    \end{subfigure}
    \caption{\small Comparing path histograms and trajectories of 
    TPS using fixed-length two-way shooting and comparing it with our variational approach.
    }
    \label{fig:mueller-path-histograms}
\end{minipage}\hspace*{.04\textwidth}
\begin{minipage}{.48\textwidth}
    \centering
    \begin{subfigure}[t]{0.47\textwidth}
        \includegraphics[width=0.9\linewidth]{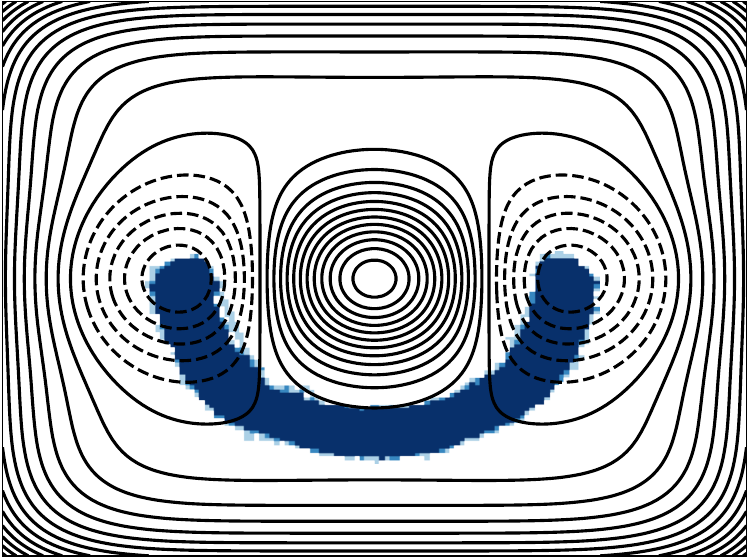}
        \caption{Single Gaussian}
        \label{subfig:toy-gaussian-single}
    \end{subfigure}
    \begin{subfigure}[t]{0.47\textwidth}
        \includegraphics[width=0.9\linewidth]{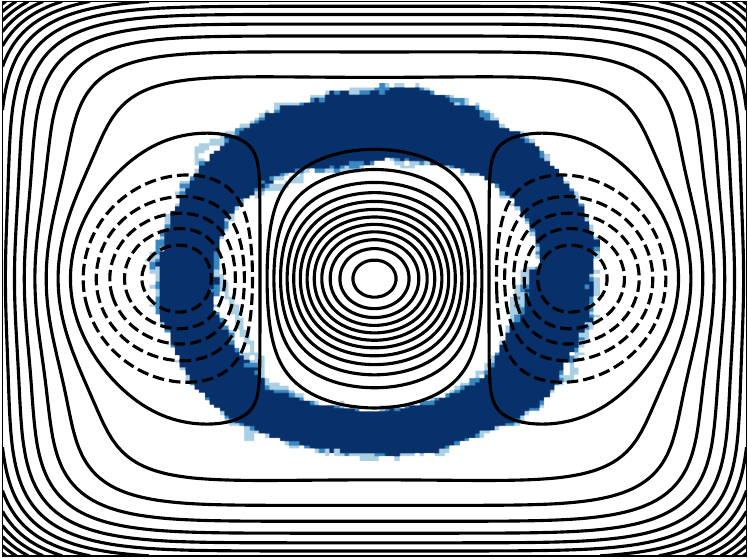}
        \caption{Mixture of Gaussians}
        \label{subfig:toy-gaussian-mixture}
    \end{subfigure}
    \caption{\small Illustration of the expressivity of unimodal Gaussian versus mixture of Gaussian paths on a symmetric potential with two transition path modes.}
    \label{fig:paths-with-and-without-mixture}
\end{minipage}
\end{figure*}
\begin{table}[t]
\resizebox{\columnwidth}{!}{
\large
\begin{tabular}{l|c|cc|cc} \hline
Method & \# Evaluations ($\downarrow$) & Max Energy ($\downarrow$) & MinMax Energy ($\downarrow$) & Log-Likelihood ($\uparrow$) & Max Log-Likelihood ($\uparrow$)   \\\hline 
MCMC (variable) & 3.53M & -13.77 $\pm$ 16.43 & -40.75 & - & - \\ 
MCMC & 1.03B & -17.80 $\pm$ 14.77 & -40.21 & 866.56 $\pm$ 17.00 & 907.15  \\
Ours & \textbf{1.28M} & -14.81 $\pm$ 13.73 & -40.56 & 858.50 $\pm$ 17.61 & 909.74  \\
\bottomrule
\end{tabular}
}
\vspace{0.2cm}
\caption{Transition path sampling experiment for Müller-Brown potential. We report the number of potential evaluations needed to sample 1,000 paths, as well as the maximum energy and the likelihood of each path (including mean and standard deviation). The methods marked with `variable' use a variable length setting. MinMax energy reports the lowest maximum energy across all paths (i.e., energy of lowest transition state).}
\label{table:muller}
\secvheader
\vspace*{-.4cm}
\end{table}

\secvheader 
\subsection{Synthetic Datasets}
\secvheader 

\paragraph{Müller-Brown Potential} The Müller-Brown potential is a popular benchmark to study transition path sampling between metastable states. It consists of three local minima, and we aim to sample transition paths connecting state $A$ and state $B$ with a circular state definition. In \Cref{fig:mueller-path-histograms}, we visualize the potential and the sampled paths and can see that the same ensemble is sampled {for both our method and two-way shooting}. Our method exhibits a slightly reduced variance for unlikely transitions.
In \Cref{table:muller}, we can observe that MCMC-based methods require many potential evaluations to achieve a good result, which comes from the low acceptance rate (especially when fixing the lengths of trajectories). Our method requires fewer energy evaluations (1 million vs. 1 billion) while finding paths with similar energy and likelihood. 
Note that the likelihood for variable approaches has been omitted, as it is governed by the number of steps in the trajectory and cannot be compared directly. 

\paragraph{Gaussian Mixture} We further consider a potential in which the states are separated by a symmetric high-energy barrier that allows for two distinct reaction channels. In \cref{fig:paths-with-and-without-mixture}, we observe that a single Gaussian path cannot model a system with multiple modes of transition paths. Nevertheless, this issue can be resolved using a mixture of Gaussian paths, with slightly increased computational cost.

\paragraph{The Case for Neural Networks}
According to our empirical study, the neural network parameterization of the Gaussian distribution statistics $\mu_{t_m|0,T}, \Sigma_{t_m|0,T}$ is an 
invaluable
part of our framework.
As an ablation, we consider parameterizing $\mu_{t_m|0,T}, \Sigma_{t_m|0,T}$ as piecewise linear splines whose intermediate points are updated using the same gradient-based optimizer as used for neural network training.  In \cref{app:splines},
we report results comparing the W1 distance of learned marginals using neural network versus spline parameterizations, observing that splines yield inferior results even after an order of magnitude more potential function evaluations.   We thus conclude that spline parameterizations are not competitive for learning transition paths and continue to focus on our neural-network approach.

\begin{table}
\vspace*{-.3cm}
\centering
\large
\resizebox{0.8\columnwidth}{!}{
\begin{tabular}{l|c|c|cc} \hline
Method & States & \# Evaluations ($\downarrow$) & Max Energy ($\downarrow$) &  MinMax Energy ($\downarrow$) \\\hline 
MCMC (variable length)& CV & 21.02M & 740.70 $\pm$ 695.79 & 52.37  \\
MCMC* & CV & 1.29B* & 288.46 $\pm$ 128.31 & 60.52  \\
\midrule
MCMC (variable length) & relaxed & 187.54M & 412.65 $\pm$ 334.70 & 26.97 \\
MCMC & relaxed & > 10B & N/A & N/A \\
\midrule
MCMC (variable length) & exact & > 10B & N/A & N/A  \\
MCMC & exact & > 10B & N/A & N/A  \\
Ours (Cartesian) & exact & \textbf{38.40M} & 726.40 $\pm$ 0.07  & 726.18   \\
Ours (Cartesian, 2 Mixtures) & exact  & 51.20M & 709.38 $\pm$ 162.37 & 513.72 \\ 
Ours (Cartesian, 5 Mixtures) & exact  & 51.20M & 541.26 $\pm$ 278.20 & 247.96  \\
Ours (Internal) & exact & \textbf{38.40M} & -14.62 $\pm$ 0.02 & -14.67  \\
Ours (Internal, 2 Mixtures)& exact & 51.20M & -15.38 $\pm$ 0.14 & -15.54\\
Ours (Internal, 5 Mixtures)& exact & 51.20M &  -15.50 $\pm$ 0.31 & \textbf{-15.95}  \\
\bottomrule
\end{tabular}
}
\vspace{0.2cm}
\caption{
Transition path sampling for alanine dipeptide.
For MCMC methods, we compare different state definitions of $\mathcal{A}, \mathcal{B}$: `CV' uses $\phi, \psi$ angles.  `Exact' uses a very small threshold of aligned root-mean-square deviation (RMSD) around reference states $A,B$ (as in Ours). `Relaxed' uses a larger threshold of RMSD around $A,B$. The method marked with a * only samples 100 paths due to computational limitations, while others sample 1,000.  Fields with N/A are intractable as a single trajectory requires more than 1 billion potential evaluations. }
\label{table:aldp}
\secvheader 
\secvheader
\vspace*{-.2cm}
\end{table}
\begin{figure}[t]
\vspace*{-.05cm}
    \centering
    \includegraphics[width=0.65\linewidth]{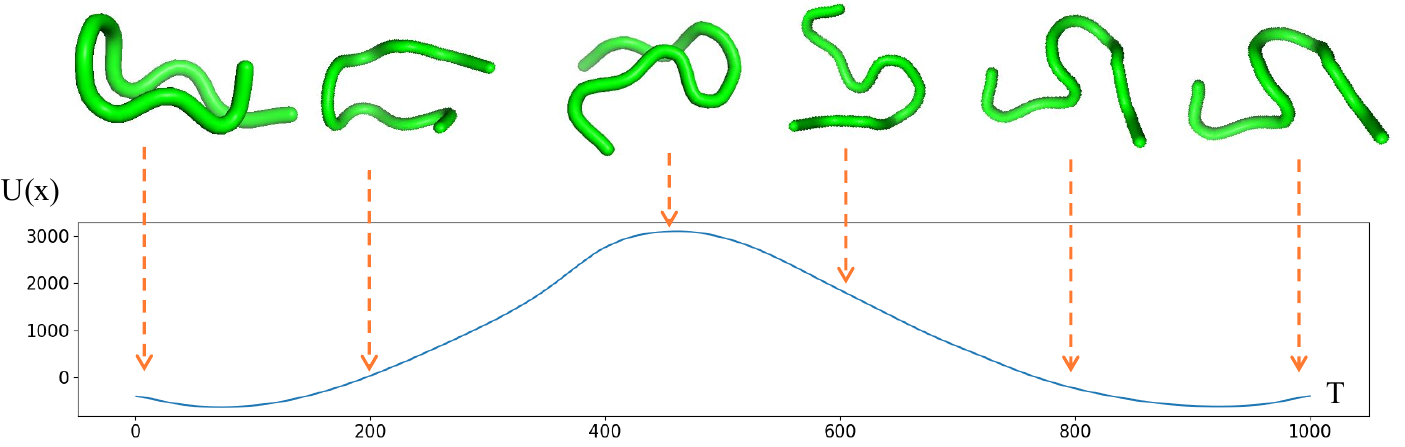}
    \caption{Transition path for the protein Chignolin. 
    The energy plot a transition path in which the protein folds in $T=1,000$ fs, and passes a high energy barrier at $460fs$ with about $3,000 $ kJ/mol.}
    \label{fig:chignolin}
    \vspace*{-0.4cm}
\end{figure}

\secvheader
\subsection{Second-order Dynamics and Molecular Systems}
\secvheader 

\paragraph{Experiment Setup} We evaluate our methods on real-world high-dimensional molecular systems governed by the second-order dynamics \labelcref{eq:second_order}: \textit{alanine dipeptide} and \textit{Chignolin}. Alanine dipeptide is a well-studied system of 22 atoms (66 total degrees of freedom), where the molecule can be described by two collective variables (CV): the dihedral angles $\phi$, $\psi$. Chignolin is a larger system consisting of 10 residues with 138 atoms (414 total degrees of freedom) that cannot be summarized as easily. We use an AMBER14 force field~\citep{maier2015ff14sb} implemented in OpenMM~\citep{eastman2017openmm} but use DMFF~\citep{wang2023dmff} to backpropagate through the energy evaluations. 

\paragraph{Alanine Dipeptide} In \Cref{table:aldp}, we report results for four variants of our models, which either predict Cartesian coordinates or internal coordinates in the form of bond lengths and dihedral angles (compare \Cref{appendix:molecular-systems}), either with or without Gaussian mixture.  
{For our method, operating in internal coordinates yields better results compared to Cartesian coordinates, where the internal coordinates are distributed similarly to Gaussians and our network does need not learn equivariances \citep{du2022se}.} Similarly, Gaussian mixture paths perform slightly better than a single Gaussian path due to the additional expressiveness. We note that paths sampled with Gaussian mixture exhibit a larger variance in max energy as they represent multiple reaction channels. 

We find that prior-informed definitions of the desired initial and target states (i.e., CV) are necessary for \gls{MCMC} to work efficiently with fixed-length trajectories. Finding these CVs in practice is challenging and only possible in this instance because the molecule is small and well-studied. For the larger system size in \cref{table:aldp}, it becomes intractable to use \gls{MCMC} to connect precise states $A, B$ (`exact') instead of larger regions (`relaxed'), even with a single trajectory. {Variable length \gls{MCMC} with relaxed endpoint conditions with CV perform well on this task, but our method is competitive using fewer evaluations and more strict boundary conditions. Fixed-length \gls{MCMC}, even with prior-informed knowledge, can only find 100 trajectories while needing 50 times more potential evaluations compared to variable length. }

\paragraph{Chignolin} The folding dynamics of Chignolin already pose a challenge and have not yet been well-studied compared to alanine dipeptide. We illustrate the qualitative experimental results for this system in \Cref{fig:chignolin}. Operating in Cartesian space, our model samples a feasible transition within 25.6M potential energy evaluation calls and a transition with a duration of $T=1ps$.
\secvheader 
\vspace*{-.05cm}
\section{Conclusion, Limitations and Future Work}
\label{sec:conclusion}
\secvheader 
In this paper, we propose an efficient computational framework for transition path sampling with Brownian dynamics. We formulate the transition path sampling problem by using Doob's $h$-transform to condition a reference stochastic process, and propose a variational formulation for efficient optimization. Specifically, we propose a simulation-free training objective and model parameterization that imposes boundary conditions as hard constraints. We compare our method with MCMC-based baselines and show comparable accuracy with lower computational costs on both synthetic datasets and real-world molecular systems. Our method is currently limited by rigidly defining states A and B to be a point mass with Gaussian noise instead of any arbitrary set.   Finally, our method might be improved by accommodating variable length paths.

\secvheader 
\vspace*{-.1cm}
\begin{ack}
\secvheader 

The authors would like to thank Juno Nam and Soojung Yang for spotting the unphysical steric barrier in the original pair of initial and target states, Jungyoon Lee for spotting an error in the energy computation, and Guan-Horng Liu, Maurice Weiler, Hannes Stärk and Yanze Wang for helpful discussions. The work of Yuanqi Du and Carla P. Gomes was supported by the Eric and Wendy Schmidt AI in Science Postdoctoral Fellowship, a Schmidt Futures program; the National Science Foundation (NSF), the Air
Force Office of Scientific Research (AFOSR); the Department of Energy; and the Toyota Research Institute (TRI). The work of Michael Plainer was supported by the Konrad Zuse School of Excellence in Learning and Intelligent Systems (ELIZA) through the DAAD program Konrad Zuse Schools of Excellence in Artificial Intelligence, sponsored by the German Ministry of Education and Research and by Max Planck Society.
The work of Kirill Neklyudov was supported by IVADO.
\end{ack}

\bibliographystyle{apalike}
\bibliography{refs}

\clearpage
\appendix

\section{Proofs}
\label{appendix:proof}

\subsection{Proofs from \texorpdfstring{\cref{sec:doobs-h-transform}}{Sec. \ref{sec:doobs-h-transform}} (Doob's \texorpdfstring{$h$}{h}-Transform Background)}\label{app:doob}
\jamison*
\begin{proof}
    See \citet{jamison1975markov} for a simple proof based on Ito's Lemma, assuming smoothness and strict positivity of $h$.
\end{proof}

\doobpdes*
\begin{proof}
    Let $p(x_{t+s} = y\cond x_t=x)$ denote the transition probability of a reference diffusion process
\begin{equation}
  \resizebox{\textwidth}{!}{\ensuremath{  \dderiv{}{s}p(x_{t+s} = y\cond x_t=x) = -\inner{\nabla_y}{p(x_{t+s} = y\cond x_t=x) b_{t+s}(y)} + \sum\limits_{ij} (\Dt)_{ij} \dderiv{^2}{y_i\partial y_j}p(x_{t+s} = y\cond x_t=x),}}
\end{equation}
where $(\Dt)_{ij} = \frac{1}{2}\Xi_{t+s}\Xi_{t+s}^T$.

Now we condition the process on the end-point value $x_T \in \setB$, and we get another kernel, i.e.
\begin{align}
    p(x_{t+s}=y\cond x_t=x, x_T \in \setB) = \frac{p(x_T \in \setB\cond x_{t+s} = y)}{p(x_T \in \setB\cond x_{t} = x)}p(x_{t+s} = y\cond x_t=x)\,.
\end{align}
We let $\hset(x,t) = p(x_T \in \setB \cond x_t = x)$  denote the conditional probability over the desired endpoint condition given $x_t=x$.   
According to laws of conditional probability, we can describe how $\hset(x,t)$ changes in time using the unconditioned transition probability
\begin{align}
    \underbrace{p(x_T \in \setB\cond x_{t} = x)}_{\hset(x,t)} = \int dy\; \underbrace{p(x_T \in \setB\cond x_{t+s} = y)}_{\hset(y,t+s)}p(x_{t+s} = y\cond x_t=x)\,,
\end{align}

we take the derivative $\deriv{}{s}$ on both sides, and we get
\begin{align}
    0 = \int dy\; \left[p(x_{t+s} = y\cond x_t=x)\deriv{\hset(y,t+s)}{s} + \deriv{p(x_{t+s} = y\cond x_t=x)}{s}\hset(y,t+s)\right]\,.
\end{align}
Using the FP equation for the transition probability and integrating by parts, we have
\small 
\begin{align}
    0 = \int dy\; p(x_{t+s} = y\cond x_t=x)\left[\deriv{\hset(y,t+s)}{s} + \inner{\nabla_y \hset(y,t+s)}{b_t(y)} + \sum_{ij} (\Dt)_{ij} \deriv{^2}{y_i\partial y_j}\hset(y,t+s)\right]\,. \nonumber
\end{align}
\normalsize
Note that this holds $\forall x$, hence, we have
\begin{align}
    \deriv{\hset(y,t+s)}{s} + \inner{\nabla_y \hset(y,t+s)}{b_{t+s}(y)} + \sum_{ij} (\Dt)_{ij} \deriv{^2}{y_i\partial y_j}\hset(y,t+s) = 0\,, \nonumber
\end{align}
without any loss of generality we can set $t=0$
\begin{align}
    \deriv{\hset(y,s)}{s} + \inner{\nabla_y \hset(y,s)}{b_{s}(y)} + \sum_{ij} (\Dt)_{ij} \deriv{^2}{y_i\partial y_j}\hset(y,s) = 0\,. \label{eq:h_doob_pf}
\end{align}
as desired to prove the optimality condition in \cref{eq:doob_hjb}.

To prove \cref{eq:doob_fp}, denote $p(y,s) = p(x_s=y\cond x_0=x)$ and differentiate $p(x_{s}=y\cond x_0=x, x_T \in \setB) = \frac{\hset(y,s)}{\hset(x,0)}p(y,s)$ as
\begin{align*}
    \deriv{}{s}&p(x_{s}=y\cond x_0=x, x_T \in \setB) \\
    = ~&\frac{1}{\hset(x,0)}\left[p(y,s)\deriv{\hset(y,s)}{s} + \hset(y,s)\deriv{p(y,s)}{s}\right] \\
    = ~&\frac{1}{\hset(x,0)}\bigg[-\inner{\nabla_y \hset(y,s)}{p(y,s)b_{s}(y)} - p(y,s)\sum_{ij} (\Dt)_{ij} \deriv{^2}{y_i\partial y_j}\hset(y,s) \\
    &\qquad~ - \hset(y,s)\inner{\nabla_y}{p(y,s) b_{s}(y)} + \hset(y,s)\sum_{ij} (\Dt)_{ij} \deriv{^2}{y_i\partial y_j}p(y,s)\bigg] \nonumber\\
    = ~&-\inner{\nabla_y}{\frac{\hset(y,s)}{\hset(x,0)}p(y,s)b_s(y)} - p(y,s)\inner{\nabla_y}{2D\nabla_y \frac{\hset(y,s)}{\hset(x,0)}}  \\ 
    ~&\qquad~\pm \inner{\nabla_y p(y,s)}{2D\nabla_y \frac{\hset(y,s)}{\hset(x,0)}} +\sum_{ij} (\Dt)_{ij} \deriv{^2}{y_i\partial y_j}\left(\frac{\hset(y,s)}{\hset(x,0)}p(y,s)\right)\,, \nonumber
\end{align*}
Note that $\hset(x,0)$ can be pulled outside the differential operator because it is a function of $x$.
The PDE for the new kernel $p(y,s\cond \setB) = p(x_{s}=y\cond x_0=x, x_T \in \setB)$ (conditioned on the end-point) becomes
\begin{align}
    \deriv{}{s}p(y,s\cond \setB) =
    -\inner{\nabla_y}{p(y,s\cond \setB)\left(b_s(y) + 2D\nabla_y \log \hset(y,s)\right)} + \sum_{ij} (\Dt)_{ij} \deriv{^2}{y_i\partial y_j}p(y,s\cond \setB)\,.
\end{align}
which matches the desired PDE in \cref{eq:doob_fp} thereby proving the first two parts of \cref{prop:doob_pdes}.

Finally, to show \cref{eq:doob_shjb}, we index time using $t$ in \cref{eq:h_doob_pf} and change variables $\hset(x,t) = e^{s(x,t)}$,
\begin{align*}
    \deriv{e^{s(x,t)}}{t} + \inner{\nabla_x e^{s(x,t)}}{b_{t}(x)} + \sum_{ij} (\Dt)_{ij} \deriv{^2}{x_i\partial x_j}e^{s(x,t)} &= 0\,. \\
    e^{s(x,t)}\deriv{{s(x,t)}}{t} + e^{s(x,t)} \inner{\nabla_x s(x,t)}{b_{t}(x)} + \inner{\nabla
    }{\Dt \nabla e^{s(x,t)}} &= 0
\end{align*}  
Next, we simplify $ \inner{\nabla
    }{\Dt \nabla e^{s(x,t)}} =  \inner{\nabla
    }{\Dt  e^{s(x,t)} \nabla s(x,t)} = \inner{\nabla e^{s(x,t)}
    }{\Dt  \nabla s(x,t)} +  e^{s(x,t)} \inner{\nabla}{\Dt \nabla s(x,t)} =  e^{s(x,t)} \inner{\nabla s(x,t)
    }{\Dt  \nabla s(x,t)} +  e^{s(x,t)} \inner{\nabla}{\Dt \nabla s(x,t)} $ to finally write
    \begin{equation}
        e^{s(x,t)} \left( \deriv{{s(x,t)}}{t} +  \inner{\nabla_x s(x,t)}{b_{t}(x)}
       +  \inner{\nabla s(x,t)
    }{\Dt  \nabla s(x,t)} +  \sum_{ij} (\Dt)_{ij} \deriv{^2}{x_i\partial x_j} s(x,t) 
 \right)= 0 \nonumber
    \end{equation}
    which demonstrates \cref{eq:doob_shjb} since the inner term must be zero. 
\end{proof}

\subsection{Proofs from \texorpdfstring{\cref{sec:lagrangian}}{Sec. \ref{sec:lagrangian}} (Lagrangian Action Minimization for Doob's \texorpdfstring{$h$}{h}-Transform)}
\label{app:doobs_lagrangian}
\newcommand{\myh}{\mathfrak{h}_t}

We begin by proving \cref{th:minmax}, whose proof actually contains the initial steps needed to prove our main theorem \cref{th:var_doob}.   In both proofs,  we omit conditioning notation $q_t \gets q_\condt$ for simplicity and assume  $q_t(x) s_t(x) \rightarrow 0$ vanishes at the boundary $x \rightarrow \pm \infty $, which is used when integrating by parts in $x$.

\minmax*
\begin{proof}
Consider the following action functional
\begin{align*}
    \mathcal{S} =~& \min_{q_t,v_t} \int dt\;\int dx\; q_t(x) \inner{v_t(x)}{\Dt v_t(x)}\,,\\
    ~&\text{ s.t. } \;\; \deriv{q_t(x)}{t} = - \inner{\nabla_x}{q_t(x)\left(b_t(x) + 2\Dt v_t(x)\right)} + \sum_{ij} (\Dt)_{ij} \deriv{^2}{x_i\partial x_j}q_t(x)\,,\\
    ~&\phantom{\text{ s.t. } \;\; } q_0(x) = \delta(x-A), \;\; q_1(x) = \delta(x-B)\,.
\end{align*}
The Lagrangian of this optimization problem is
\begin{align}
    \mathcal{L} =  \int_0^T dt\;\int dx\; \left[q_t\inner{v_t}{\Dt v_t} + s_t\left(\deriv{q_t}{t} + \inner{\nabla}{q_t\left(b_t + 2\Dt v_t\right)} - \sum_{ij} (\Dt)_{ij} \deriv{^2}{x_i\partial x_j}q_t\right)\right]\,, \nonumber
\end{align}
where $s_t$ is the dual variable and we omit the optimization arguments, with $\mathcal{S}= \min_{q_t,v_t} \max_{s_t} \mathcal{L}$.
Swapping the order of optimizations under strong duality, we take the variation with respect to $v_t$ in an arbitrary direction $\myh$. Using $\Dt = \Dt^T$, we obtain
\begin{align}
    \frac{\delta \mathcal{L}}{\delta v_t}[\myh] &= 
     ~ q_t \inner{(\Dt + \Dt^T)v_t}{\myh} - q_t  \inner{2\Dt^T\nabla s_t}{\myh} = 0 \nonumber \\
   & \implies v_t = \nabla s_t\,, \label{eq:v_nabla_s}
\end{align}
Substituting into the above, we have
\begin{align}
    \mathcal{L} = \int_0^T dt\;\int dx\; \left[s_t\deriv{q_t}{t} - q_t\inner{\nabla s_t}{\Dt \nabla s_t} + s_t\inner{\nabla}{q_t b_t} - 
    s_t \inner{\nabla }{\Dt \nabla q_t}
   \right]\,. \label{eq:lagr2}
\end{align}
Integrating by parts in $t$ and in $x$, assuming that $q_t(x) s_t(x) \rightarrow 0$ as $x \rightarrow \pm \infty $,  yields
\small
\begin{align}
    \mathcal{L} &= \int dx ~  q_T s_T - \int dx~ q_0 s_0  + \int_0^T dt\;\int dx\; \left[-q_t \deriv{s_t}{t} - q_t\inner{\nabla s_t}{\Dt \nabla s_t} - q_t\inner{\nabla s_t }{b_t} + \inner{\nabla  s_t }{\Dt \nabla q_t} \right] \nonumber \\
    &= \int dx ~  q_T s_T - \int dx~ q_0 s_0  + \int_0^T dt\;\int dx\; \left[-q_t \deriv{s_t}{t} - q_t\inner{\nabla s_t}{\Dt \nabla s_t} - q_t\inner{\nabla s_t }{b_t} -  q_t \inner{\nabla  }{\Dt \nabla s_t } \right] \nonumber \\
    &=  \int dx ~  q_T s_T - \int dx~ q_0 s_0  -\int_0^T dt\;\int dx\; q_t \left[ \deriv{s_t}{t} +\inner{\nabla s_t}{\Dt \nabla s_t} +\inner{\nabla s_t }{b_t} +   \inner{\nabla  }{\Dt \nabla s_t } \right]
    \label{eq:x_and_t}
\end{align}
\normalsize
where in the second line, we integrate by parts in $x$ again.  
Enforcing $q_T(x) = \delta(x - B)$ and $q_0(x) = \delta(x-A)$ and recalling $\mathcal{S}= \min_{q_t} \max_{s_t} \mathcal{L}$ after eliminating $v_t$, we recover the optimization in the statement of the corollary.
\end{proof}

\begin{theorem*}\textup{\textbf{1.}}
\label{th:var_doob_app}
The following Lagrangian action functional has a unique solution which matches the Doob $h$-transform in \cref{prop:doob_pdes},
\begin{subequations}
\label{eq:doob_lagrangian_all_app}
\begin{align}
\hspace*{-.3cm}    &\mathcal{S} =~ \min_{q_\condt, v_\condt} \int_0^T dt\;\int dx\; q_\condt(x) \inner{v_\condt(x)}{\Dt ~  v_\condt(x)}\,, \label{eq:doob_lagrangian_app} \\[1.25ex]
   &\resizebox{.91\textwidth}{!}{\text{s.t.}~~ \ensuremath{
   \dderiv{q_\condt(x)}{t} = - \inner{\nabla_x}{q_\condt(x)\left(b_t(x) + 2\Dt~ v_\condt(x)\right)} + \mathlarger{\sum}\limits_{ij} (\Dt)_{ij} \dderiv{^2}{x_i\partial x_j}q_\condt(x), 
   }}
   \label{eq:doob_lagrangian_fp_app} \\
    &~\phantom{\deriv{q_\condt}{t}\hspace{.016\textwidth}}  q_0(x) = \delta(x-A), \qquad q_T(x) = \delta(x-B)\,.\label{eq:doob_lagrangian_fp_boundary_app}
\end{align}
\end{subequations}
Namely, the optimal $q_\condt^*(x)$ obeys \labelcref{eq:doob_fp} and the optimal $v_\condt^*(x) = \nabla_{\ifbm{x}} \log \hset(\ifbm{x},t) = \nabla_{\ifbm{x}} s(\ifbm{x},t)$ follows \labelcref{eq:doob_hjb} or \labelcref{eq:doob_shjb}.
\end{theorem*}
\begin{proof}
    The proof proceeds from \labelcref{eq:lagr2} above,
    \begin{equation}
   \resizebox{.93\textwidth}{!}{\ensuremath{  \mathcal{S} = \min\limits_{q_t}\max\limits_{s_t} \mathcal{L} = \min\limits_{q_t}\max\limits_{s_t} \int_0^T dt\;\int dx\; \left[s_t\deriv{q_t}{t} - q_t\inner{\nabla s_t}{\Dt \nabla s_t} + s_t\inner{\nabla}{q_t b_t} - 
    s_t \inner{\nabla }{\Dt \nabla q_t}
   \right]\,.
   }}
   \label{eq:lagr}
   \end{equation}
We first show that the optimality condition with respect to $s_t$ yields the Fokker-Planck equation for $q_t$ in \cref{prop:doob_pdes} \labelcref{eq:doob_fp}, before deriving the PDE in \labelcref{eq:doob_hjb} as the optimality condition with respect to $q_t$.

\textit{Optimality Condition for \labelcref{eq:doob_lagrangian_all_app} recovers \cref{prop:doob_pdes} \labelcref{eq:doob_fp}:}
The variation with respect to $s_t$ of \labelcref{eq:lagr} is simple, apart from the intermediate term.   For a perturbation direction $\myh$, we seek
\begin{align}
   \int dx\; \frac{\delta (\sbullet)}{\delta s_t} ~ \myh = \frac{d}{d\eps}\left[- \int dx\;  q_t\inner{\nabla \left( s_t + \eps \myh \right)}{\Dt \nabla \left( s_t + \eps \myh \right)}\right]  \Big|_{\eps=0}  , \nonumber
\end{align}
where $\sbullet$ indicates the functional on the right hand side.  Proceeding to differentiate with respect to $\eps$, we use linearity to pull $\frac{d}{d\eps}$ inside the integral and apply it first to obtain $\frac{d}{d\eps}(s_t + \eps \myh) = \myh$.  Using the product rule, recognizing the symmetry of terms, and evaluating at $\eps = 0$, we are left with
\begin{align}
   \int dx\; \frac{\delta (\sbullet)}{\delta s_t} ~ \myh = \left[- 2 \int dx\;  q_t\inner{\nabla \myh }{\Dt \nabla s_t}\right] \overset{(i)}{=} \left[ \int dx\;  \myh \Big( 2 \inner{\nabla  }{q_t \Dt \nabla s_t} \Big) \right] \label{eq:hard_term}
\end{align}
where in $(i)$ we integrate by parts $x$.   

We are now ready to set the variation of \labelcref{eq:lagr} with respect to $s_t$ (in an arbitrary direction $\myh$) equal to zero.   Using \labelcref{eq:hard_term}, we have
\begin{align} 
    \frac{\delta \mathcal{L}}{\delta s_t}[\myh] =  0 &= \deriv{q_t}{t} + 2\inner{\nabla  }{q_t \Dt \nabla s_t} + \inner{\nabla}{q_t b_t} - 
    \inner{\nabla }{\Dt \nabla q_t} \nonumber  \\
    \implies ~~ 0 &=  \deriv{q_t}{t} + \inner{\nabla  }{q_t \Big( b_t + 2 \Dt \nabla s_t\Big)} - 
    \inner{\nabla }{\Dt \nabla q_t} 
\end{align}
which matches the desired optimality condition for the conditioned process in \cref{prop:doob_pdes} \labelcref{eq:doob_fp}.

\textit{Optimality Condition for \labelcref{eq:doob_lagrangian_all_app} recovers \cref{prop:doob_pdes} \labelcref{eq:doob_hjb}:}
Starting again from \labelcref{eq:lagr}, we take the variation with respect to $q_t$.   First, we repeat identical steps (integrate by parts in both $x$ and $t$) to reach \labelcref{eq:x_and_t},
\begin{align}
        \mathcal{L}  &=  \int dx ~  q_T s_T - \int dx~ q_0 s_0  -\int_0^T dt\;\int dx\; q_t \left[ \deriv{s_t}{t} +\inner{\nabla s_t}{\Dt \nabla s_t} +\inner{\nabla s_t }{b_t} +   \inner{\nabla  }{\Dt \nabla s_t } \right]
  \nonumber
\end{align}
where it is now clear that taking the variation with respect to $q_t$ and setting equal to zero yields
\begin{align}
\frac{\delta \mathcal{L}}{\delta q_t}[\myh] =  0 = \deriv{s_t}{t} +\inner{\nabla s_t}{\Dt \nabla s_t} +\inner{\nabla s_t }{b_t} +   \inner{\nabla  }{\Dt \nabla s_t } \label{eq:b8}
\end{align}
which is the desired PDE for $s(x,t) = \log \hset(x,t)$ in \labelcref{eq:doob_shjb}.   To obtain \labelcref{eq:doob_hjb}, we note an identity used to simplify the last term
\small 
\begin{align}
    \sum_{ij} (\Dt)_{ij} \deriv{^2}{x_i\partial x_j} \log h_t &= \inner{\nabla}{\Dt \nabla \log h_t} = \inner{\nabla}{\frac{1}{h_t}\Dt\nabla h_t} = -\frac{1}{h_t^2}\inner{\nabla h_t}{\Dt \nabla h_t} + \frac{1}{h_t} \inner{\nabla }{\Dt \nabla h_t}. \nonumber
\end{align}\normalsize
Now, substituting $s(x,t) = \log \hset(x,t)$  into \cref{eq:b8} and abbreviating $\log \hset(\sbullet, t) =  \log h_t(\sbullet)$, we obtain
\begin{align}
    \frac{1}{h_t}\deriv{h_t}{t} + \frac{1}{h_t^2}\inner{\nabla h_t}{\Dt \nabla h_t} + \frac{1}{h_t}\inner{\nabla h_t}{b_t} -\frac{1}{h_t^2}\inner{\nabla h_t}{\Dt \nabla h_t} + \frac{1}{h_t} \inner{\nabla }{\Dt \nabla h_t} &= 0\,, \nonumber \\
\implies   \quad  \deriv{h_t(x)}{t} + \inner{\nabla h_t(x)}{b_t(x)} + 
     \inner{\nabla }{\Dt \nabla h_t} &= 0,
\end{align}
which matches \labelcref{eq:doob_hjb} as desired.

The last equation defines the backward Kolmogorov equation for the diffusion process with the drift $b_t(x)$ and covariance matrix $\Dt$, i.e. the function $h_t(x)$ defines the conditional density $h_t(x) = p(x_T \in \setB' \cond x_t = x)$ for some set $\setB'$, which agrees with the forward process with the same drift and covariance.
The boundary condition $q_T(x) = \delta(x-B)$ together with the backward equation define the unique solution to this PDE.
Since the PDEs and the boundary conditions are the same as in Doob's $h$-transform, we have $h_t(x) = p(x_T = B \cond x_t = x)$.
\end{proof}

\sb*
\begin{proof}
    We use the Girsanov theorem \citep[Sec. 7.3]{sarkka2019applied} to calculate the KL divergence between the following two Brownian diffusions with fixed initial condition $x_0 = A$,
    \begin{align}
 \bbPref_{0:T}: \qquad \qquad 
    \ifbm{dx_t}&= \btofxt \cdot dt + \Xi_t~ \ifbm{dW_t} \,, \;\; \qquad \quad\ \qquad \qquad \qquad  \\
\bbQ^{v}_{0:T}:  \qquad \qquad
\ifbm{d\tpsx_t}&= \big(
    \btofxt + 
    2 \Dt ~  v_\condt(\ifbm{x}_\condt)
    \big)\cdot dt 
    + \Xi_t ~ \ifbm{dW_t}\,, \qquad   \label{eq:fp_q}
    \end{align}
    In particular, noting the difference of drifts is $ \btofxt + 
    2 \Dt ~  v_\condt(x_t) - \btofxt = 2 \Dt ~  v_\condt(x_t)$,  the likelihood ratio is given by
    \begin{align}
        \frac{d\bbQ^{v}_{0:T}}{d\bbPref_{0:T}} = \frac{q_\condt(x_{0}, ... x_T)}{\prob(x_{0}, ... x_T)} = \exp\Big\{  - \frac{1}{2} \int_0^T  
        \inner{2 \Dt ~  v_\condt(x_t)}{(\Dt)^{-1} ~ 2 \Dt ~  v_\condt(x_t)} dt \\
        - \int 2 \left( \Dt ~  v_\condt(x_t)\right)^T \Dt^{-1} dW_t       
        \Big\} \nonumber 
    \end{align}
We finally calculate the KL divergence, noting that, after taking the log, the expectation of the integral $\int (\sbullet) dW_t$ in the final term vanishes,
\begin{align}
D_{KL}[ \bbQ^v_{0:T} : \bbPref_{0:T} ] =  2 \int_0^T dt \; \int dx_t \; q_\condt(x_t) ~ \inner{v_\condt(x_t) }{\Dt ~ v_\condt(x_t)} ,
\end{align}
which matches \labelcref{eq:doob_lagrangian} up to a constant factor of 2 does not change the optimum.
We finally compare to the constraints in \cref{th:var_doob}.  First, it is clear that the diffusion in \labelcref{eq:fp_q} satisfies the Fokker-Planck equation in \labelcref{eq:doob_lagrangian_fp} \citep[Sec. 5.2]{sarkka2019applied}.   We respect \labelcref{eq:doob_lagrangian_fp_boundary} by optimizing over endpoint-constrained path measures, which yields
\begin{align}
   \cS = \min \limits_{\bbQ^v_{0:T} ~\text{s.t.} ~ \bbQ^v_0 = \delta_A, \bbQ^v_T = \delta_B} D_{KL}[ \bbQ^v_{0:T} : \bbPref_{0:T}  ]
\end{align}
as desired.
\end{proof}

\section{Gaussian Path Parameterizations}\label{app:gaussians}
\gaussianpath*
\begin{proof}
Consider the following identities for the Gaussian family of marginals $q_t(x) = \mathcal{N}(x | \mu_t, \Sigma_t)$, where we omit conditioning $q_t \gets q_\condt$ for simplicity of notation,
\begin{subequations}   
\begin{align}
 \hspace*{-.35cm} \log q_t(x) &= -\frac{1}{2}(x-\mu_t)^T\Sigma_t^{-1}(x-\mu_t) -\frac{d}{2}\log(2\pi)- \frac{1}{2}\log \det \Sigma_t\,, \\
    \nabla_x \log q_t(x) &= -\Sigma_t^{-1}(x-\mu_t)\,,\\
    \deriv{}{t} \log q_t(x) &= (x-\mu_t)^T\Sigma_t^{-1}\deriv{\mu_t}{t} +\frac{1}{2}(x-\mu_t)^T\Sigma_t^{-1}\deriv{\Sigma_t}{t}\Sigma_t^{-1}(x-\mu_t) -\frac{1}{2}
    {\text{tr}\left(\Sigma_t^{-1} \deriv{\Sigma_t}{t}\right)}\label{eq:dlogqdt}
\end{align}\label{eq:gaussian_identities}
\end{subequations}

\newcommand{\uode}{u^{\text{o}}_t}
We begin by solving for a vector field $\uode(x)$ that satisfies the continuity equation (where $\uode$ denotes the drift of an \textsc{ODE})
\begin{align}
    \deriv{q_t}{t} = -\inner{\nabla_x}{q_t \uode} &= - q_t \inner{\nabla_x}{ \uode} + \inner{\nabla_x q_t}{\nabla_x \uode}   \, \nonumber \\
    \implies   \deriv{}{t}\log q_t &= -\inner{\nabla_x}{\uode} - \inner{\nabla_x \log q_t}{\uode} \label{eq:log_prop}
\end{align}
The vector field satisfying this equation is 
\begin{align}
    \uode(x) = \deriv{\mu_t}{t} + \frac{1}{2}\deriv{\Sigma_t}{t}\Sigma_t^{-1}(x-\mu_t) 
\end{align}
which we can confirm using the identities in \cref{eq:gaussian_identities}.  
Indeed, for the terms on the RHS of \cref{eq:log_prop}, 
\begin{align*}
    -\inner{\nabla_x}{\uode} =~& -\frac{1}{2}\text{tr}\left(\Sigma_t^{-1}\deriv{\Sigma_t}{t}\right)\,, \\
    -\inner{\nabla_x \log q_t}{\uode} =~& \inner{\Sigma_t^{-1}(x-\mu_t)}{\deriv{\mu_t}{t}} + \frac{1}{2}(x-\mu_t)^T\Sigma_t^{-1}\deriv{\Sigma_t}{t}\Sigma_t^{-1}(x-\mu_t)\,.
\end{align*}
Putting these terms and the time derivative from \cref{eq:dlogqdt} into \cref{eq:log_prop} we conclude the proof.

However, we are eventually interested in finding the formula for the drift $u_t$ that satisfies the Fokker-Planck equation in \labelcref{eq:fpe_general}.
That is, to describe the same evolution of density $\deriv{q_t(x)}{t}$, the relationship between $u_t$ and $\uode$ is as follows
\begin{align*}
      \deriv{q_t(x)}{t}=  -\inner{\nabla_x}{q_t \uode} =~& - \inner{\nabla_x}{q_t ~ u_t} + \inner{\nabla_x}{ \Dt \nabla_x q_t} \\
     =~& - \inner{\nabla_x}{q_t ~ u_t} + \inner{\nabla_x}{ \Dt q_t \nabla_x \log q_t} \\
     =~& - \inner{\nabla_x}{q_t \underbrace{\left( u_t - \Dt \nabla_x \log q_t\right)}_{\uode}}
\end{align*}
Finally, we use the identities in \cref{eq:gaussian_identities} to obtain
\begin{align*}
    u_t = ~\uode + \Dt \nabla_x \log q_t =~& \deriv{\mu_t}{t} + \frac{1}{2}\deriv{\Sigma_t}{t}\Sigma_t^{-1}(x-\mu_t) - \Dt \Sigma_t^{-1}(x-\mu_t) \\
 \implies \quad u_t  = ~& \deriv{\mu_t}{t} + \left[\frac{1}{2}\deriv{\Sigma_t}{t}\Sigma_t^{-1} - \Dt \Sigma_t^{-1}\right](x-\mu_t)
\end{align*}

\end{proof}

\mixture*
\begin{proof}
    See \citet{peluchetti2023diffusion} Theorem 1 and its proof in their \cref{appendix:proof}.   
\end{proof}
\section{Extended Related Work}

\subsection{Machine Learning for Molecular Simulation}

The main dilemma of molecular dynamics comes from the accuracy and efficiency trade-off---accurate simulation requires solving the Schrödinger equation which is computationally intractable for large systems, while efficient simulation relies on empirical force fields which is inaccurate. Recently, there has been a surge of work in applying machine learning approaches to accelerate molecular simulation. One successful paradigm is machine learning force field (MLFF) which leverages the transferability and efficiency of machine learning methods to fit force/energy prediction models on quantum mechanical data and transfer across different atomic systems~\cite{smith2017ani,wang2018deepmd}. More recently, increasing attention has been focused on building atomic foundation models to encompass all types of molecular structures~\cite{batatia2023foundation,shoghi2023molecules,zhang2022dpa}. 

Sampling is a classical problem in molecular dynamics to draw samples from the Boltzmann distribution of molecular systems. Classical methods mainly rely on Markov chain Monte Carlo (MCMC) or MD which requires long mixing time for multimodal distributions with high energy barriers~\cite{rotskoff2024sampling}. Generative models in machine learning demonstrate promises in alleviating this problem by learning to draw independent samples from the Boltzmann distribution of molecular systems (known as Boltzmann generator)~\cite{noe2019boltzmann}. Numerous methods have been developed to utilize generative models as a proposal distribution for escaping local minima in running MCMC methods~\cite{gabrie2022adaptive}. However, one critical issue is that generative models rely on training from samples. Although recent advances have been developed to learn from unnormalized density (i.e., energy) function, the training inefficiency limits their applicability to solve high-dimensional molecular dynamics problems. To circumvent the curse of dimensionality for the sampling problem, another branch of work study to learn coarse-grained representation with neural networks~\cite{sidky2020molecular}. For broader literature of applying machine learning to enhanced sampling, we refer the reader to~\citet{mehdi2024enhanced}.
\section{Further Experimental Details}
\label{appendix:exp_details}

\subsection{Evaluation Metrics}

To assess the quality of our approach in terms of performance and physicalness of paths, we compare them under different metrics to well-established TPS techniques. One important describing factor of a trajectory is the molecule's highest energy during the transition. These high-energy states are often referred as transition states and less likely to occur but they determine importance factors during chemical reaction such as reaction rate. As such, we will look at the maximum energy along the transition path and use it to compare the ensemble of trajectories more efficiently. The main goal is to show that lower energy of the transition states can be sampled by the methods.

However, the maximum energy does not account for the fact that the transition path needs to be sequential, and each step needs to be coherent based on the previous position and momentum. For this, we also compare the likelihood of the paths (i.e., unnormalized density) by computing the probably of being in the start state \(\rho(x_0)\) and multiplying it with the step probability such that
\begin{align}
    L \left( x_0, x_1, \ldots, x_{N-1} \right) = \rho(x_0) \cdot \prod_{i=0}^{N-2} \pi \left( x_{i+1} \cond x_i \right) \,.
\end{align}
For the step probability \(\pi\), we solve the Langevin leap-frog implementation as implemented in OpenMM to solve \(\Normal(x_{i+1} \cond x_i + dt\cdot b_{t_i}(x),dt\sigma_i^2)\). As for the starting probability, we compute the unnormalized density of the Boltzmann distribution for our start state \(z\) and assume that the velocity \(v\) can be sampled independently~\citep[Sec.~4.6]{castellan1983physical}
\begin{align}
    \rho(z, v) \propto \exp \left( -\frac{U(z)}{k_BT} \right) \cdot \Normal \left( v \cond 0, k_B T \cdot M^{-1}\right)\,,
\end{align}
with the Boltzmann constant \(k_B\) and the diagonal matrix \(M\) containing the mass of each atom. 

As for the performance, the number of energy evaluations will be the main determining factor of the runtime for larger molecular systems, especially for proteins. We hence compare the use of the number of energy computations as a proxy for hardware-independent relative measurements. In our tests, this number aligned with the relative runtime of these approaches.  

\subsection{Toy Potentials}
The toy systems move according to the following integration scheme (first-order Euler)
\begin{align}
    x_{t+1} = x_t -dt\cdot \nabla_x U(x_t) + \sqrt{dt} \cdot \texttt{diag}(\xi) \cdot \varepsilon, \quad \varepsilon \sim \Normal(0,1)\,,
\end{align}
following the definition of our stochastic system in \cref{sec:doobs-h-transform} with a time-independent Wiener process, where $\xi$ is a constant time-independent standard deviation for all dimensions.

\paragraph{Müller-Brown}
The underlying Müller-Brown potential that has been used for our experiments can be written as
\begin{equation}
    \begin{aligned}
    U(x, y) = & -200 \cdot \exp \left( -(x-1)^2 -10 y^2 \right) \\
              & -100 \cdot \exp \left( -x^2 - 10 \cdot (y - 0.5)^2 \right) \\
              & -170 \cdot \exp \left( -6.5 \cdot (0.5 + x)^2 + 11 \cdot (x +0.5) \cdot (y -1.5) -6.5 \cdot (y -1.5)^2 \right) \\
              & + 15 \cdot \exp \left( 0.7 \cdot (1 + x)^2 +0.6 \cdot (x + 1) \cdot (y -1) +0.7 \cdot (y -1)^2 \right) \, .
    \end{aligned}
\end{equation}
We used a first-order Euler integration scheme to simulate transition paths with 275 steps and a $dt$ of $10^{-4}s$. $\xi$ was chosen to be 5 and 1,000 transition paths were simulated. We have used an MLP with four layers and a hidden dimension of 128 each, with swish activations. It has been trained for 2,500 steps with a batch size of 512 and a single Gaussian. 

In \cref{subfig:mueller-path-density}, we compare the likelihood of the sampled paths. We can see that one-way shooting takes time until the path is decorrelated from the initial trajectory, which is shorter and thus has a higher likelihood. All MCMC methods exhibit this behavior, which is typically alleviated by using a warmup period in which all paths are discarded. After that, all methods exhibit similar likelihood, with our method having a slightly lower likelihood. Looking at the transition state (i.e., maximum energy on the trajectory) in \cref{subfig:mueller-max-energy} reveals that all methods have a similar quality of paths.

\begin{figure}[htbp]
    \centering
    \begin{subfigure}[t]{0.3\textwidth}
        \includegraphics[width=\linewidth]{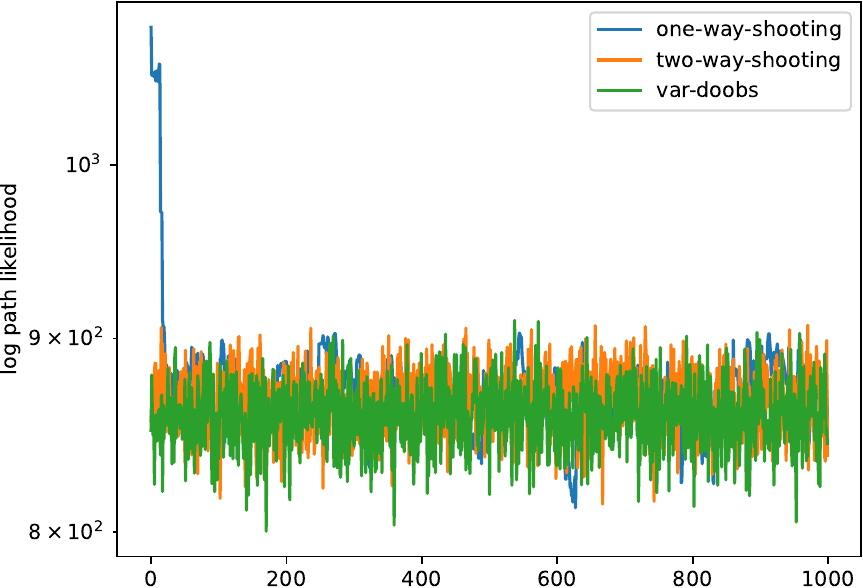}
        \caption{Log Path Likelihood}
        \label{subfig:mueller-path-density}
    \end{subfigure}
    \hspace{1cm}
    \begin{subfigure}[t]{0.3\textwidth}
        \includegraphics[width=\linewidth]{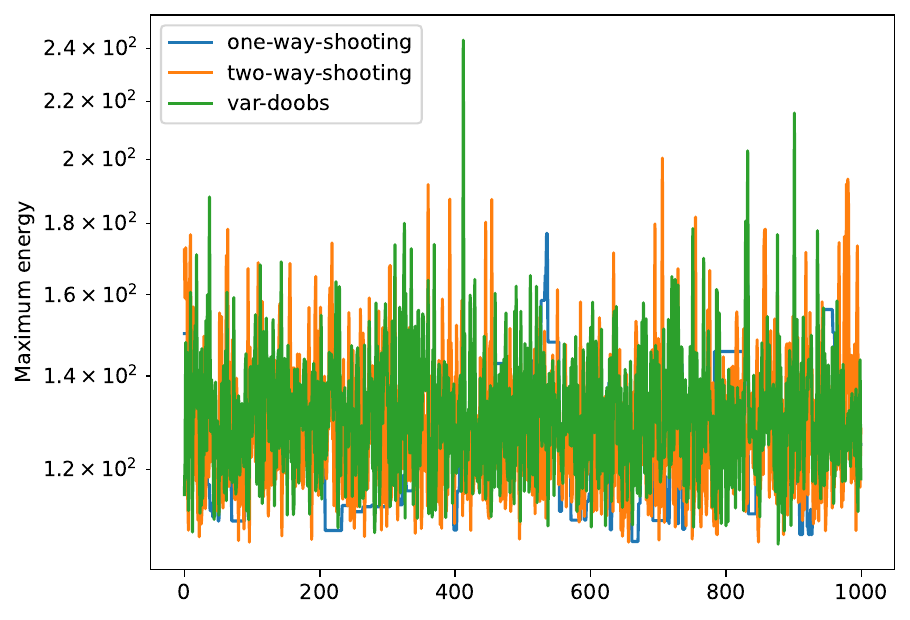}
        \caption{Maximum Energy}
        \label{subfig:mueller-max-energy}
    \end{subfigure}
    \caption{In \subref{subfig:mueller-path-density}, we compare the log likelihood of sampled trajectories, where a higher likelihood is generally more favorable.
    The plot in \subref{subfig:mueller-max-energy} shows the maximum energy of each individual trajectory. A high maximum energy means that the molecule needs to be in an excited state during the transition, making it less likely to occur under lower temperatures.
    }
\end{figure}

We can further analyze the quality of our method by investigating the difference between the \enquote{ground truth} marginal $\condp_{\pcondt}(\tpsx)$ and the learned marginal $q_\condt(x)$. For this, we compute the Wasserstein W1 distance~\citep{flamary2021pot} between the marginal observed by fixed-length two-way shooting (which we assume to be close to the ground truth) and our variational approach. We observe a mean W1 distance of $0.130 \pm 0.026$ and visualize it along the time coordinate $t$ (in steps) in \cref{fig:w1-distance}.

\begin{figure}[htbp]
    \centering
    \includegraphics[width=0.3\linewidth]{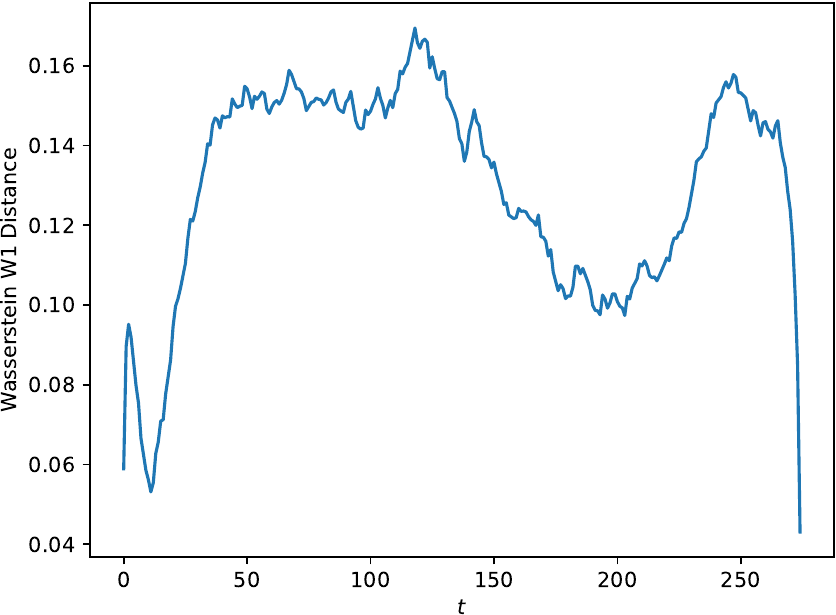}
    \caption{In this figure, we compare the Wasserstein W1 distance between the marginals. The densities are almost identical at the beginning and end state. Note that at the third local minimum of the Müller-Brown potential along the trajectory (i.e., reached at about $200$ steps), the marginals align more closely as well.
    }
    \label{fig:w1-distance}
\end{figure}

\paragraph{Dual-Channel Double-Well} To demonstrate the advantage of mixtures, we have used the two-dimensional potential
\begin{equation}   
    \begin{aligned}
    U(x, y) = & + 2 \cdot \exp \left( -(12 x^2 + 12 y^2) \right) \\
              & - 1 \cdot \exp \left( -(12 \cdot (x + 0.5)^2 + 12 y^2) \right) \\
              & - 1 \cdot \exp \left( -(12 \cdot (x - 0.5)^2 + 12 y^2) \right) + x^6 + y^6\, .
    \end{aligned}
\end{equation}
In this case, we have used $dt=5 * 10^{-4}s$ with a transition time of $T=1s$ and $\xi = 0.1$. As for the MLP, we have used the same structure as in the Müller-Brown example but trained it for 20,000 iterations. The corresponding weights to \cref{prop:mixture} are $w = [ \frac{1}{2}, \frac{1}{2} ]$ and are fixed for this experiment and hence $w \not\in \theta$. 

\subsection{Neural Network Ablation Study} \label{app:splines}
In \cref{fig:w1-distance-spline-nnet} we compare how different parameterizations of $\mu^{(\theta)}_\condt$, and  $\Sigma^{(\theta)}_\condt$ impact the quality of trajectories on the Müller-Brown potential. For this, we compare linear and cubic splines (with 20 knots) with neural networks. As a metric to estimate the quality, we compare the Wasserstein W1 distance between the learned marginal $q_\condt(x)$ and the marginal observed by fixed-length two-way shooting (i.e., baseline). We notice that using linear splines results in the highest W1 distance, while cubic splines improve the quality. Using neural networks, however, yields the best approximation.

We have fixed the computational budget for all systems, which means that we have trained splines for more epochs than the neural network (since they are slower to train). For high-dimensional systems, the runtime is mostly determined by the number of potential evaluations and not the complexity of the architecture. We thus conclude that the additional expressivity provided by neural networks is necessary for more complicated (molecular) systems and does not introduce much computational overhead.

\begin{figure}[htbp]
    \centering
    \includegraphics[width=0.3\linewidth]{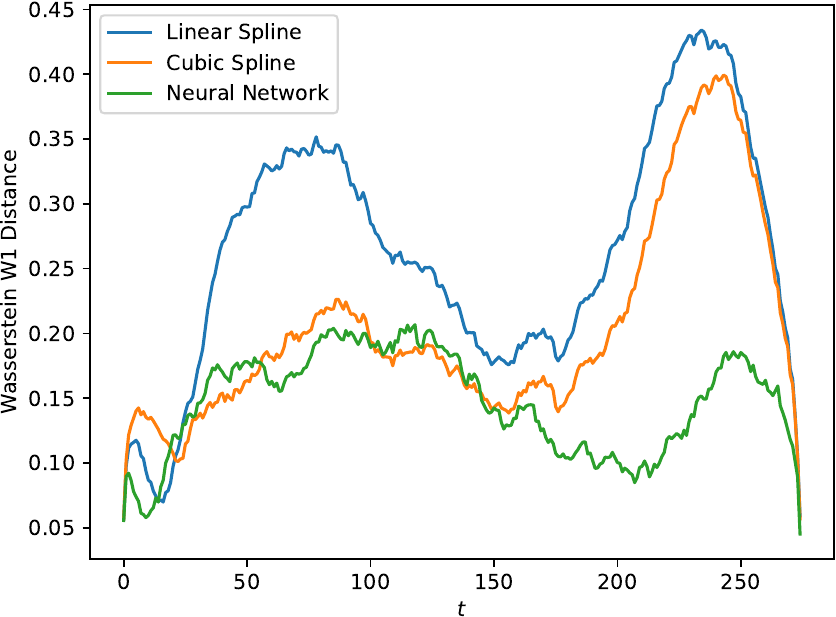}
    \caption{We compare the Wasserstein W1 distance between the marginals with different parameterization techniques of $\mu^{(\theta)}_\condt$, and  $\Sigma^{(\theta)}_\condt$. 
    }
    \label{fig:w1-distance-spline-nnet}
\end{figure}

\subsection{Molecular Systems} \label{appendix:molecular-systems}
To simulate molecular dynamics, we rely on the AMBER14 force field ($\text{amber14/protein.ff14SB}$~\cite{maier2015ff14sb}) without a solvent, as implemented in OpenMM \citep{eastman2017openmm}. As OpenMM does not support auto-differentiation, we do not use OpenMM for the simulations themselves, but utilize DMFF~\citep{wang2023dmff} which is a differentiable framework implemented in JAX~\citep{jax2018github} for molecular simulation. This is needed because during training we compute $\nabla_\theta U\left(x_\condt \sim \mathcal{N} (\mu^{(\theta)}_\condt, \Sigma^{(\theta)}_\condt) \right)$, where the concrete $x_\condt$ is sampled based on the parameters of the neural network. 

For the concrete simulations, we ran them with the timestep $dt = 1fs$, with $T=1ps$, $\gamma=1ps$, and $\text{Temp} = 300K$.
To compute the MCMC two-way shooting baselines, we use the same settings and consider trajectories as failed, if they exceed 2,000 steps without reaching the target.

\paragraph{Neural Network Parameterization}
We parameterize our model with neural networks, a 5-layer MLP with ReLU activation function and 256/512 hidden units for alanine dipeptide and Chignolin, respectively. The neural networks are trained using an Adam optimizer with learning rate $10^{-4}$.

We represent the molecular system in two ways: (1) in Cartesian coordinates, which are the 3D coordinates of each atoms, and with (2) internal coordinate which instead uses bond length, angle and dihedral angle along the molecule, where we use the same parameterization as in~\citep{noe2019boltzmann}.

Our state definition includes a variance parameter for the initial and target marginal distributions at $t=0$ and $t=T$, we choose the variance to be $10^{-8}$ which almost does not change the energy of the perturbed system.

\paragraph{Visualization of Transition for Alanine Dipeptide} In \cref{fig:aldp-transition-path}, we show a transition sampled without any noise from the model with internal coordinates and 2 Gaussian mixtures.

\begin{figure}[htbp]
    \centering
    \includegraphics[width=1\linewidth]{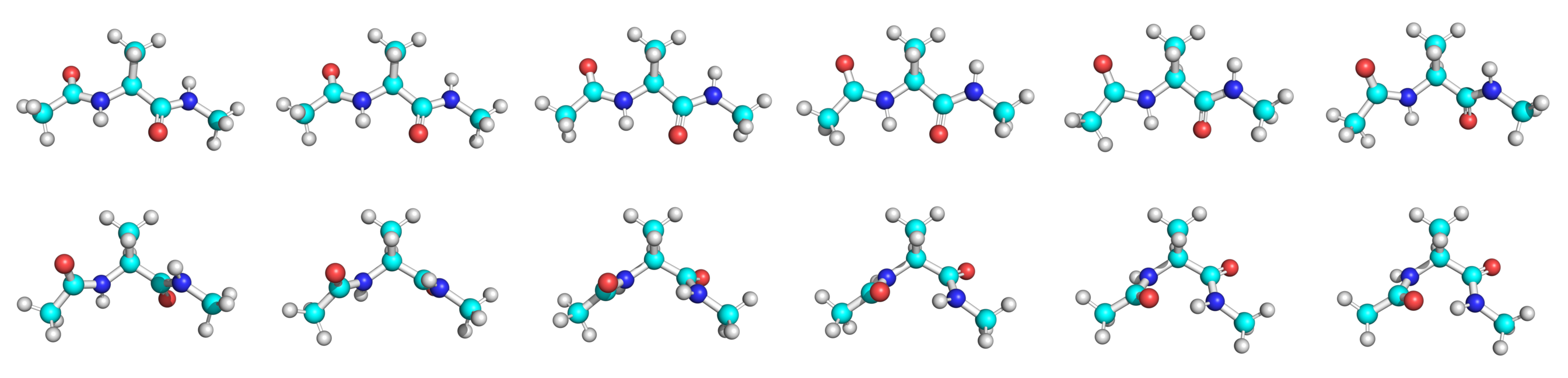}
    \caption{Transition path for the alanine dipeptide.}
    \label{fig:aldp-transition-path}
\end{figure}

\paragraph{Comparison of Sampled Paths During Training} In our training procedure, the marginal starts with a linear interpolation between $A$ and $B$, which produces very unlikely paths with potentially high energy states. In \cref{fig:aldp-energy-by-iteration}, we compare how the quality of paths changes depending on the number of training iterations (i.e., the number of potential evaluations). We show the curve for a single Gaussian mixture with Cartesian coordinates. Similar trends can be observed in other settings.

\begin{figure}[htbp]
    \centering
    \includegraphics[width=0.3\linewidth]{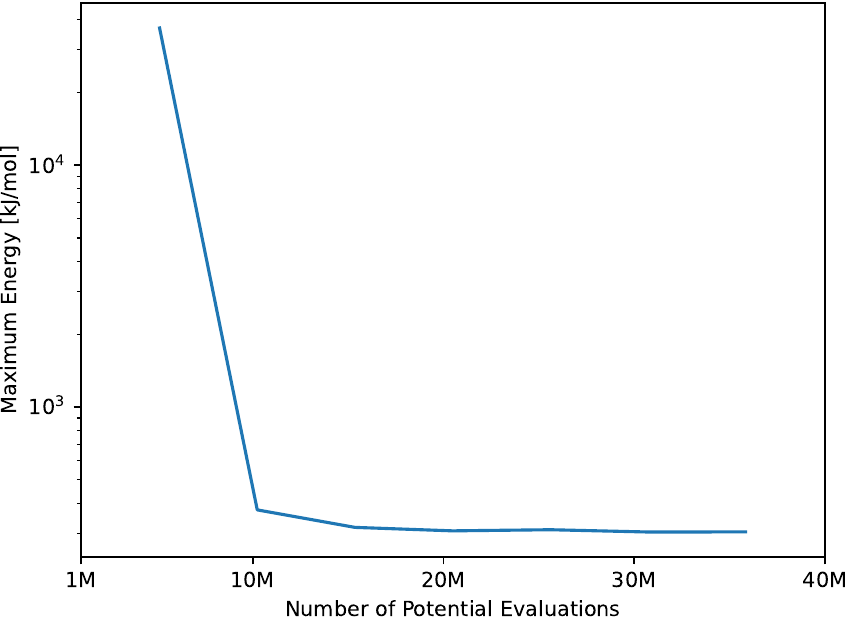}
    \caption{In this figure, we compare the quality of paths based on the current training step (i.e., potential evaluations). We observe that with increasing training time, paths with higher likelihood are sampled.}
    \label{fig:aldp-energy-by-iteration}
\end{figure}

\paragraph{Loss Curves} In this section, we would like to investigate the training losses of different configurations. For this, we plot the exponential moving average of the loss ($\alpha=0.001$) to better highlight the trends of the noisy variational loss. \cref{fig:aldp-loss} compares the results of different training settings. We can observe that mixtures can decrease the overall loss, but all model variations converge to a similar loss value.  

\begin{figure}[htbp]
    \centering
    \includegraphics[width=0.3\linewidth]{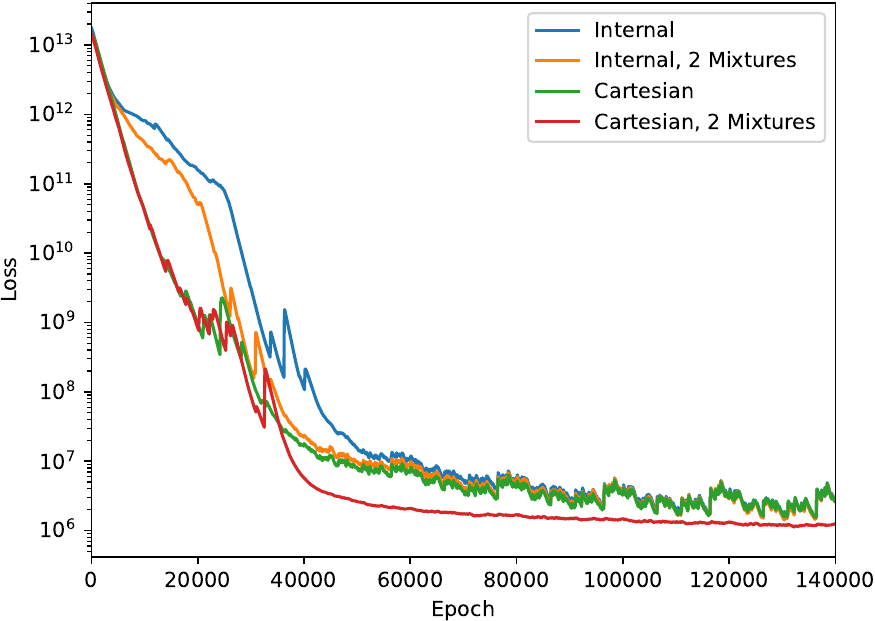}
    \caption{Visualization of the loss for different training setups. These setups are identical to what has been reported in \Cref{table:aldp}.}
    \label{fig:aldp-loss}
\end{figure}

\subsection{Computational Resources} \label{appendix:computational-resources}

All our experiments involving training were conducted on a single NVIDIA A100 80GB. The baselines themselves were computed on a M3 Pro 12-core CPU. 
\section{Societal Impact}
\label{appendix:social_impact}

Our research concerns the efficient sampling of transition paths which are crucial for a variety of tasks in biology, chemistry, materials science and engineering. Our research could potentially benefit research areas from combustion, catalysis, protein design to battery design. Nevertheless, we do not foresee special potential negative impacts to be discussed here. 


\end{document}